\documentclass{article} 
\usepackage{iclr2027_conference,times}

\usepackage{amsmath,amsfonts,bm}

\def\eqref#1{equation~\ref{#1}}

\def\1{\bm{1}}

\DeclareMathAlphabet{\mathsfit}{\encodingdefault}{\sfdefault}{m}{sl}
\SetMathAlphabet{\mathsfit}{bold}{\encodingdefault}{\sfdefault}{bx}{n}

\newcommand{\E}{\mathbb{E}}

\newcommand{\R}{\mathbb{R}}

\usepackage{dcolumn}

\usepackage{hyperref}
\usepackage{url}

\usepackage{amsmath,amssymb,amsfonts,amsthm,mathtools,bm}

\usepackage{graphicx}
\usepackage{xcolor}

\usepackage{booktabs}
\usepackage{tabularx}
\usepackage{array}
\usepackage{multirow}
\usepackage{makecell}

\usepackage{float}
\usepackage{subcaption}

\definecolor{gain}{HTML}{138A45}
\newcolumntype{G}{>{\color{gain}\scriptsize}l}

\newcommand{\RowNorm}{\mathcal{R}}
\newcommand{\Polar}{\mathcal{P}}
\newcommand{\Orient}{\mathcal{O}}
\newcommand{\NSfive}{\operatorname{NS5}}

\newcommand{\inner}[2]{\left\langle #1,#2\right\rangle}
\newcommand{\Frob}[1]{\left\lVert #1\right\rVert_{\mathrm F}}
\newcommand{\spec}[1]{\left\lVert #1\right\rVert_2}
\newcommand{\rowmax}[1]{\left\lVert #1\right\rVert_{2,\infty}}

\newtheorem{theorem}{Theorem}
\newtheorem{lemma}{Lemma}
\newtheorem{proposition}{Proposition}

\floatstyle{ruled}
\newfloat{algorithm}{tbp}{loa}
\floatname{algorithm}{Algorithm}

\newcounter{algline}
\newcommand{\ResetAlgorithmLines}{\setcounter{algline}{0}}

\newcommand{\AlgLine}[2][0]{%
  \refstepcounter{algline}%
  \par\noindent%
  \hangindent=\dimexpr 2.35em + #1em\relax%
  \hangafter=1%
  \makebox[1.8em][r]{\scriptsize\thealgline}%
  \hspace{0.55em}\hspace*{#1em}#2\par%
}

\title{Scaling Muon for Diffusion Transformers}

\author{%
\begin{tabular}{@{}c@{}}
{\small\bfseries
Chenghao Li$^{1,2}$\thanks{Work done during an internship at Meta.},\ 
Xiao Han$^{2}$,\ 
Xinxin Huang$^{2}$,\ 
Wei Liu$^{2}$,\ 
Boyang Li$^{2}$,\ 
Bing Xiao$^{2}$,}
\\[-1pt]
{\small\bfseries
Heran Zhang$^{2}$,\ 
Juanma Perez Rua$^{2}$,\ 
Ke Xu$^{2}$,\ 
Kangning Liu$^{2}$,\ 
Linjun Kuang$^{2}$,\ 
Na Li$^{2}$,}
\\[-1pt]
{\small\bfseries
Tan Wang$^{2}$,\ 
Tian Xie$^{2}$,\ 
Wei Peng$^{2}$,\ 
Yang Pei$^{2}$,\ 
Yifan Xu$^{2}$,\ 
Yuanhao Zhai$^{2}$,}
\\[-1pt]
{\small\bfseries
Yuwei Lin$^{2}$,\ 
Zhe Wang$^{2}$,\ 
Zihao He$^{2}$,\ 
Daniel Li$^{2}$,\ 
Junbiao Tang$^{2}$,\ 
Ziyang Jiang$^{2}$,\ 
Dake Chen$^{2}$}
\\[3pt]
{\small\normalfont
$^{1}$University of Southern California
\qquad
$^{2}$Meta}
\end{tabular}%
}

\iclrfinalcopy 
\begin{document}

\maketitle

\fancyhead{}
\renewcommand{\headrulewidth}{0pt}

\begin{abstract}
The matrix-aware optimizer Muon improves large model training by balancing updates across singular directions, yet its scaling behavior and end-to-end efficiency on large Diffusion Transformers (DiTs) remain unclear. We first establish Muon's scaling behavior on DiTs from 1.3B to 15B parameters, showing that its optimization and generative quality advantages over AdamW persist across model scales. However, at scale, the 5-step Newton--Schulz iteration (NS5) performed at every optimization step, together with full-momentum materialization, introduces substantial computation and communication overhead that can offset Muon's step-efficiency advantage. We introduce \emph{Periodic Row-wise Muon}, which performs a full NS5 spectral update once every \(K\) steps and applies a low compute and communication cost row-wise constrained update based on the current momentum at the remaining steps. We further co-design a distributed implementation that operates directly on sharded momentum during non-refresh steps and accelerates spectral refreshes through bucketed all-gather and communication--computation overlap. Across all scales, Muon improves the best observed generative quality over AdamW by 12.9--19.1\%. Compared with vanilla Muon, Periodic Row-wise Muon remains within 0.5\% in best generative quality on the 1.3B--4B models and improves it by 4.5\% at 9B. It reduces optimizer time by 46.9--54.3\%, end-to-end step time by 15.7--24.3\%, and logical communication volume by 66.7\%, while reaching its respective best generative quality with 33.7--64.8\% less active training time. These results show that Periodic Row-wise Muon preserves Muon's generative quality advantage while translating it into end-to-end training efficiency for large DiTs.

\end{abstract}

\section{Introduction}
\label{section:intro}

The continued scaling of Diffusion Transformers (DiTs)~\citep{peebles2023scalable,esser2024scaling,chen2024pixart} has improved model capacity and generative quality, while making training efficiency a central concern. Prior work has characterized predictable improvements with increased training compute~\citep{yin2025towards,liang2026scaling}, 
but has typically treated the optimizer as fixed.
At billion-parameter scale, however, optimizer efficiency can no longer be judged solely by the number of steps required to reach a lower loss. Fewer optimization steps 
should
translate into 
lower GPU-hours and shorter wall-clock time under distributed training. 
Evaluating optimizer scalability therefore requires considering both its optimization advantage and its realized compute and communication cost under distributed training.

Muon~\citep{jordan2024muon} presents a promising alternative for improving large-scale training efficiency. Unlike AdamW~\citep{loshchilov2017decoupled}, which applies coordinate-wise adaptive updates, Muon treats two-dimensional weight tensors as matrices and applies a finite-step Newton--Schulz (NS) transformation to their momentum, producing update directions with global spectral structure. 
Recent work has shown that, when properly calibrated, Muon can match AdamW with substantially fewer training FLOPs in large language model pretraining and post-training~\citep{liu2025muon}. 
Evidence from smaller diffusion models likewise suggests improved optimization despite a higher per-step compute and communication cost, while showing that loss, generative quality, and runtime may rank optimizers differently~\citep{schaipp2025optimization}. It therefore remains unclear whether Muon's optimization advantage persists as DiTs scale to ten-billion parameters and whether its 
compute and communication cost
can be reduced sufficiently to improve end-to-end training efficiency.

Answering these questions requires considering optimization and compute and communication costs jointly. Under the sharded execution path studied in this work, Muon's NS5 transformation introduces additional matrix multiplications, full-momentum communication, synchronization, and temporary materialization~\citep{liu2025muon}. 
These considerations motivate reducing the frequency of full spectral transformations and the associated exposed compute and communication, while preserving Muon's optimization advantage.

We first establish that Muon's optimization advantage persists at scale by training DiTs from 1.3B to 15B parameters on the GPIC dataset~\citep{chandrasegaran2026gpic}. Across all four scales, Muon achieves lower validation loss and improves the best observed FD-DINO over AdamW by 12.9--19.1\%. 
However, vanilla Muon requires substantially more computation and communication per step.
These results demonstrate Muon's algorithmic scalability on large DiTs while exposing a systems bottleneck to realizing its optimization advantage efficiently.

To address this bottleneck, we revisit whether global spectral geometry must be imposed at every optimization step and introduce \emph{Periodic Row-wise Muon}, 
which performs a full NS5 update once every \(K\) steps, and applies a row-wise constrained normalization operator (RowNorm)~\citep{lau2026symmetry} 
at the remaining steps.
The method alternates between two complementary matrix geometries. Periodic spectral (refresh) steps provide global coupling across rows and singular directions, whereas low compute and communication cost row-wise constrained (non-refresh) steps provide local scale control and maintain stable optimization behavior between consecutive spectral refreshes.


We co-design the distributed execution with the periodic update. On non-refresh steps, RowNorm operates directly on sharded momentum and for matrix shapes whose normalization rows span ranks, only their norm statistics are all-reduced instead of the full matrix. On refresh steps, we pipeline bucketed momentum all-gathers for upcoming matrices with NS computation on already available buckets to reduce exposed communication latency. This design removes full-momentum materialization on non-refresh steps and reduces the exposed compute and communication of the remaining spectral refreshes.



Across the 1.3B--4B models, Periodic Row-wise Muon's best observed FD-DINO remains within 0.5\% of vanilla Muon, while outperforming it by approximately 4.5\% at 9B and
2.7\% at 15B. Relative to vanilla Muon, it reduces optimizer time by 46.9--54.3\%, end-to-end step time by 15.7--24.3\%, and logical optimizer-communication volume by 66.7\% across all scales.

Our main contributions are as follows: \textbf{(1)} we characterize Muon on 1.3B--15B DiTs, showing that its advantages over AdamW in validation loss and generative quality persist across scale while identifying NS5 computation and full-momentum communication as its main systems bottlenecks; \textbf{(2)} we introduce Periodic Row-wise Muon, which replaces most NS5 transformations with low compute and communication cost RowNorm updates from the current momentum; and \textbf{(3)} we develop a distributed implementation that operates on sharded momentum during non-refresh steps and pipelines bucketed all-gather with NS5 computation during refresh steps, substantially reducing systems overhead while retaining generative quality comparable to vanilla Muon.

\section{Muon for DiTs}
\label{section:background}


\subsection{Scaling Muon for DiTs Training}
\label{section:background_optimizers}

Given a clean data sample \(x\), a diffusion timestep \(a\), and noise \(\epsilon\), let \(x_a\) denote the noisy input constructed according to the prescribed noise schedule. We train a DiT to predict the corresponding velocity target \(v^\star\) by minimizing
\begin{equation}
\label{eq:dit_objective}
\mathcal{L}(\theta)
=
\mathbb{E}_{x,a,\epsilon}
\left[
\left\lVert
v_\theta(x_a,a)-v^\star(x_a,a)
\right\rVert_2^2
\right].
\end{equation}
We omit additional conditioning variables, such as text embeddings, for notational simplicity. 
The AdamW baseline applies AdamW to all trainable parameters. Muon instead applies matrix-aware updates to two-dimensional hidden weight matrices, while biases, normalization parameters, and other non-matrix parameters continue to be updated by AdamW.

Consider a weight matrix at training step \(t\) denotes \(W_t\in\mathbb{R}^{m\times n}\) with gradient \(G_t=\nabla_{W_t}\mathcal{L}(\theta_t)\). 
Muon applies a finite spectral transformation consisting of \(J_{\mathrm{NS}}=5\)  Newton--Schulz iterations \(\Phi_5\) to its momentum \(M_t\), and decoupled weight decay update:
\begin{equation}
\label{eq:finite_ns5}
\Polar_t
=
\Phi_5(M_t),
\qquad
W_{t+1}
=
(1-\eta_t\lambda)W_t
-
\eta_t\,s(W_t)\Polar_t,
\end{equation}
where \(\eta_t\) is learning rate, \(\lambda\) is weight decay coefficient, and \(s(W_t)\) is shape-dependent update scale.

The finite map \(\Phi_5\) is inspired by the matrix polar factor and reshapes the singular values of the momentum to produce a globally coupled spectral direction \(\Polar_t\). Importantly, NS5 is not equivalent to computing an exact SVD polar decomposition, nor does it generally yield an exactly orthogonal matrix after \(5\) iterations. Our algorithmic definitions, theoretical analysis, and experiments therefore use the implemented finite map \(\Phi_5\), rather than assuming convergence to the exact polar factor.

We characterize Muon's scalability for large scale DiTs training along two dimensions: algorithmic and systems scalability. Algorithmic scalability answers whether Muon maintains stable optimization behavior as model size increases and whether its advantages over AdamW in validation loss and generative quality persist across model scales. Systems scalability answers whether these algorithmic gains translate into realized efficiency under large scale distributed training.

\subsection{Computation Complexity}
\label{section:background_compute}

The most direct computational difference between Muon and AdamW arises from updating two-dimensional matrices. Consider
\(W\in\mathbb{R}^{m\times n}\),
\(r=\min(m,n)\),
\(c=\max(m,n)\).
Because AdamW maintains and applies element-wise statistics, its arithmetic complexity for this matrix is
\(C_{\mathrm{AdamW}}(m,n)
=
\Theta(mn)
=
\Theta(rc)\).
In contrast, a finite step Newton--Schulz transformation repeatedly performs matrix multiplications. Even choosing the smaller Gram-matrix orientation yields
\(C_{\mathrm{NS}}(m,n)
=
\Theta\!\left(
J_{\mathrm{NS}}r^2c
\right)\).
Thus, NS5 incurs an additional factor of
\(\Theta(J_{\mathrm{NS}}r)\) in arithmetic per matrix.
Ignoring implementation constants and hardware throughput, their arithmetic 
complexity
ratio for one matrix is \(\Theta(J_{\mathrm{NS}}r)\).

\subsection{
Large Scale Distributed Training}
\label{section}

Local arithmetic alone does not characterize Muon's compute and communication costs at scale,
which also depends on how its full-matrix spectral transformation is
mapped onto sharded optimizer states. In our execution path, each
two-dimensional parameter and its momentum are sharded along dimension
zero, so each rank stores a subset of rows and the corresponding
momentum shard. AdamW is coordinate-separable because each rank can update its
local parameters using only its local first- and second-moment shards,
without materializing the full matrix. Muon is not shard-separable under
this execution path because Newton--Schulz iterations couple rows and
singular directions through Gram matrices and matrix multiplications.
In general, a rank cannot recover its portion of the spectral update
from its local momentum shard alone. Our vanilla implementation
therefore all-gathers the momentum shards to every rank, executes NS5
on the full matrix at every rank, retains the output slice corresponding
to the local parameter shard, and then discards the remaining output
and releases the temporary full-matrix buffers~\citep{liu2025muon}.

Muon's systems
requirements
consequently include the GEMMs and associated
matrix operations of replicated NS5 execution, together with
full-momentum communication and temporary full-matrix materialization. Suppose the sharding group contains \(p\)
ranks and each momentum element occupies \(b\) bytes. Ignoring padding
and protocol overhead, all-gathering an \(m\times n\) momentum matrix
requires each rank to logically receive approximately
\begin{equation}
\label{eq}
V_{\mathrm{AG}}(m,n)
\approx
bmn\frac{p-1}{p}
\ \text{bytes}.
\end{equation}
This quantity excludes parameter communication already required by
forward and backward propagation. Logical communication volume alone,
however, does not determine communication time, which also depends on
the collective count, bucket sizes, launch latency, the fraction of
inter-node traffic, network topology and contention, and the extent of
communication--computation overlap. In particular, many small buckets
may incur substantial launch latency without changing the total
communication volume.

Together, repeated spectral computation and full-momentum communication substantially increase Muon's per-step resource requirements and may offset its optimization advantage.
This motivates reducing the frequency of full
spectral transformations while retaining effective updates.

\section{Periodic Row-wise Muon}
\label{sec:periodic-muon}

We introduce \emph{Periodic Row-wise Muon}, which performs an NS5 spectral
update periodically and applies RowNorm to the current momentum at the
remaining steps. Rather than treating RowNorm as a numerical
approximation to NS5, our method alternates between two matrix updates
induced by distinct constraint geometries. We then exploit the locality
of RowNorm to reduce momentum collective communication and full-matrix
materialization.

\subsection{Complementary Spectral and Row-wise Geometries}
\label{sec:complementary-geometries}

For a two-dimensional momentum \(M_t\in\mathbb{R}^{m\times n}\), let
\(\widetilde M_t=\Orient(M_t)\) denote \(M_t\) if \(m\leq n\) and
\(M_t^\top\) otherwise, so that
\(\widetilde M_t\in\mathbb{R}^{r\times c}\), where
\(r=\min(m,n)\leq c=\max(m,n)\). The resulting direction is mapped
back through \(\Orient^{-1}\).

For a full-row-rank matrix \(X\in\mathbb{R}^{r\times c}\), the ideal
polar factor $\Polar(X)$ is
\begin{equation}
\Polar(X)=(XX^\top)^{-1/2}X
\in
\operatorname*{arg\,max}_{\spec{U}\leq1}
\inner{X}{U},
\qquad
\inner{A}{B}=\operatorname{tr}(A^\top B).
\label{eq:polar-geometry}
\end{equation}
The spectral-norm constraint globally couples the rows and singular
directions of \(X\). Practical Muon uses a finite NS5 map motivated by this spectral
geometry rather than computing the exact polar factor.

RowNorm, previously studied in symmetry-compatible optimizer
design~\citep{lau2026symmetry}, instead normalizes each row
independently,
\begin{equation}
\RowNorm_\epsilon(X)_{i:}
=
\frac{X_{i:}}
{\max\!\left(\lVert X_{i:}\rVert_2,\epsilon\right)},
\qquad i=1,\ldots,r.
\label{eq:rownorm}
\end{equation}
For nonzero rows, RowNorm maximizes
\(\inner{X}{U}\) subject to
\(\rowmax{U}:=\max_i\lVert U_{i:}\rVert_2\leq1\).
Thus, the ideal polar direction imposes a globally coupled spectral
constraint, whereas RowNorm preserves each momentum row's direction
while controlling its magnitude independently. The complete variational characterizations and their relation are
provided in Appendices~\ref{app:rownorm-variational}
and~\ref{app:distinct-geometries}.

Periodic spectral correction is further motivated by the local
stability of the ideal polar direction. For full-row-rank matrices
\(A\) and \(B\),
\begin{equation}
\Frob{\Polar(A)-\Polar(B)}
\leq
\frac{2\Frob{A-B}}
{\sigma_{\min}(A)+\sigma_{\min}(B)}.
\label{eq:polar-perturbation}
\end{equation}
Hence, moderate momentum changes imply moderate changes in the ideal
polar direction as long as the matrices remain away from rank
degeneracy. This motivates periodically reimposing global spectral
structure and using a lower compute and communication cost structured direction between
refreshes. The proof and limitations of this exact-polar argument are
discussed in Appendix~\ref{app:polar-stability}.

\subsection{Periodic Row-wise Muon}
\label{sec:periodic-update}

Building on the preceding analysis, we propose
\emph{Periodic Row-wise Muon}. Given a period \(K\geq 1\) and a RowNorm
multiplier \(\gamma>0\), define the refresh indicator
\(\rho_t=\mathbb{I}\!\left[t\bmod K=0\right]\).
At every step, we use the current momentum \(M_t\) to compute the parameter update
\begin{equation}
\widetilde D_t=
\begin{cases}
\NSfive(\widetilde M_t), & \rho_t=1,\\[3pt]
\gamma\,\RowNorm_\epsilon(\widetilde M_t), & \rho_t=0,
\end{cases}
\qquad
W_{t+1}=(1-\eta_t\lambda)W_t-\eta_ts(W_t)\Orient^{-1}(\widetilde D_t).
\label{eq:periodic-direction}
\end{equation}

The period \(K\) controls the frequency of spectral updates, whereas
\(\gamma\) calibrates the effective step size of the RowNorm branch
relative to the spectral branch. This calibration is necessary because
the two normalization maps correspond to different feasible sets. 
For a full-row-rank matrix with nonzero rows, the exact polar factor
and the unregularized RowNorm direction both have Frobenius norm
\(\sqrt r\). However, equal Frobenius norms do not imply equal stable
update amplitudes.
NS5 couples all rows under
spectral-norm geometry, whereas RowNorm updates rows independently under
\(\ell_{2,\infty}\) geometry. Their alignment with the gradient and the
local curvature encountered along their directions can therefore
differ. Directly reusing the Muon learning rate is equivalent to setting
\(\gamma=1\), which implicitly assumes that the two branches have the
same stable step-size range without theoretical justification.

We write the effective learning rate of an off-refresh step as
\(\eta_t^{\mathrm{RN}}=\gamma\eta_t\).
Here, \(\eta_t\) retains the global learning rate schedule of vanilla muon, while
\(\gamma\) only specifies the relative scale between the two update
geometries. Both normalization maps remove the global scale of the input
momentum, and \(s(W_t)\) already accounts for shape-dependent scaling.
The remaining calibration is therefore a dimensionless, branch-specific
quantity. We use a fixed \(\gamma\) to preserve a constant ratio between
the RowNorm and Muon learning rate schedules, rather than introducing a
separate time varying schedule for RowNorm. Complete pseudocode is provided in Appendix~\ref{app:full-optimizer}.


Under standard conditional-alignment and bounded-second-moment
assumptions, Appendix~\ref{app:conditional-descent} establishes a
finite-horizon descent bound for the periodic update. Appendix~\ref{app:blockwise-gamma} further gives the
blockwise dependence on \(\gamma\), including its interaction with the
remaining optimizer direction.

\subsection{Distributed Execution}
\label{sec:distributed-execution}


The two branches of Periodic Row-wise Muon have different communication
requirements. A refresh step must still materialize the complete
momentum matrix to execute NS5. To reduce the exposed latency of this global path, we introduce a
bucketed all-gather pipeline with
communication--computation overlap. Specifically, we partition the
Muon matrices into communication buckets. Once the momentum of the
current bucket has been all-gathered, we immediately reconstruct its
full matrices and execute NS5 while asynchronously all-gathering the
next bucket. Communication for the next bucket is thereby overlapped
with NS5 computation on the current bucket, and only a bounded
number of momentum matrices need to be materialized at any
time.

On non-refresh steps, RowNorm operates directly on the momentum shards.
If \(m\leq n\), the orientation is unchanged and every RowNorm row is
local, requiring no optimizer-specific collective. If \(m>n\), then
\(\Orient(M)=M^\top\), so each oriented row spans ranks. If rank \(p\)
owns the original rows indexed by \(\mathcal S_p\), it computes
\begin{equation}
q_j^{(p)}=\sum_{i\in\mathcal S_p}M_{ij}^2,
\qquad
\nu_j=
\max\left\{
\left(\sum_p q_j^{(p)}\right)^{1/2},\epsilon
\right\}.
\label{eq:sharded-tall-rownorm}
\end{equation}
The denominators therefore require a sum all-reduce of only \(n\)
scalars, rather than an all-gather of all \(mn\) momentum entries.
Appendix~\ref{app:sharded-rownorm-correctness} proves equivalence to dense RowNorm.

In the implementation, we first compute the local statistics for all
tall matrices, pack them, and launch bucketed asynchronous all-reduces.
While these collectives are in flight, we compute the local statistics
and RowNorm updates for the remaining matrices. Once the all-reduces
complete, we use the global statistics to finish the tall-matrix
updates. This ordering hides the statistics communication for tall
matrices behind local computation on the remaining matrices. Complete pseudocode is provide in Appendix~\ref{app:full-optimizer}

For an oriented matrix in \(\mathbb{R}^{r\times c}\), \(r\leq c\), the dominant computation
of NS5 is \(C_{\mathrm{NS5}}=\Theta(5r^2c)\), whereas RowNorm
requires only \(C_{\mathrm{RN}}=\Theta(rc)\). Let \(V_{\mathrm{AG}}\) and \(V_{\mathrm{RN}}\) denote
the total per-rank logical optimizer payloads on refresh and
non-refresh steps. The average matrix-processing computation over
one period \(\overline C(K)\) and period-averaged payload \(V(K)\) satisfy
\begin{equation}
\overline C(K)
=
\frac{1}{K}C_{\mathrm{NS5}}
+
\frac{K-1}{K}C_{\mathrm{RN}},
\qquad
\frac{V(K)}{V_{\mathrm{AG}}}
=
\frac{1}{K}
+
\frac{K-1}{K}
\frac{V_{\mathrm{RN}}}{V_{\mathrm{AG}}}.
\end{equation}
For complete-row matrices \(V_{\mathrm{RN}}=0\), while for a tall
matrix \(V_{\mathrm{RN}}/V_{\mathrm{AG}}=O(1/m)\) under the logical
payload model. For example, 
the total optimizer-specific payload for \(K=3\) approaches \(1/3\) of vanilla Muon when the norm-statistics payload is small.
Detailed arithmetic, correctness, communication, and overlap models
are given in Appendix~\ref{app:complexity-systems}.

\section{Experiments}
\label{sec:experiments}

We systematically compare AdamW, vanilla Muon, and Periodic Row-wise Muon across DiT models ranging from approximately 1.3B to 15B parameters.
Beyond the main comparison, we include all experimental results in Appendix~\ref{app:additional_experiments}.

\subsection{Experimental Setup}
\label{sec:experimental_setup}

\textbf{Dataset and models.}
We train text-to-image Diffusion Transformers from scratch on GPIC-Full~\citep{chandrasegaran2026gpic}. GPIC contains 100M training image--text pairs collected from Flickr and Wikimedia, 
captioned with Qwen3-VL-4B-Instruct~\citep{bai2025qwen3}. We train four MMDiT~\citep{peebles2023scalable} configurations at $512\times512$ resolution, containing approximately 1.3B, 4B, 9B, and 15B parameters. All models share the same overall architecture and conditioning modules, and differ only in hidden dimension, depth, and number of attention heads.

\textbf{Optimization and training.}
Unless otherwise specified, Periodic Row-wise Muon uses \(K=3\) and \(\gamma=0.15\) throughout and their selection and analysis are provided in Appendix~\ref{app:k_gamma_selection}. At each model scale, all three optimizers are trained for 60{,}000 steps with the same global batch size of 4{,}096, and therefore process the same number of training examples. All experiments use 32 nodes with 256 NVIDIA H100 GPUs in total and FSDP2 in Pytorch.
We normalize all time quantities by the
mean step time of the 1.3B AdamW run, which is defined as one unit.

\textbf{Evaluation.}
We select a fixed set of 50{,}000 prompts from the GPIC test set and generate one $512\times512$ image per prompt. All methods use the same sampling configuration with a fixed classifier-free guidance scale of 5.0. We report FD-DINOv2~\citep{stein2023exposing}, FID~\citep{heusel2017gans}, Maximum Mean Discrepancy (MMD), Precision, Recall, Density, Coverage~\citep{kynkaanniemi2019improved,sajjadi2018assessing,naeem2020reliable}, HPSv2.1~\citep{wu2023human}, and GenEval2~\citep{kamath2025geneval}. 
Completed experimental settings are provided in Appendix~\ref{app:experimental_setup}.

\begin{figure*}[t]
    \centering

    \begin{subfigure}[t]{\textwidth}
        \centering
        \includegraphics[
            width=\linewidth,
            trim={0 8pt 0 13pt},
            clip
        ]{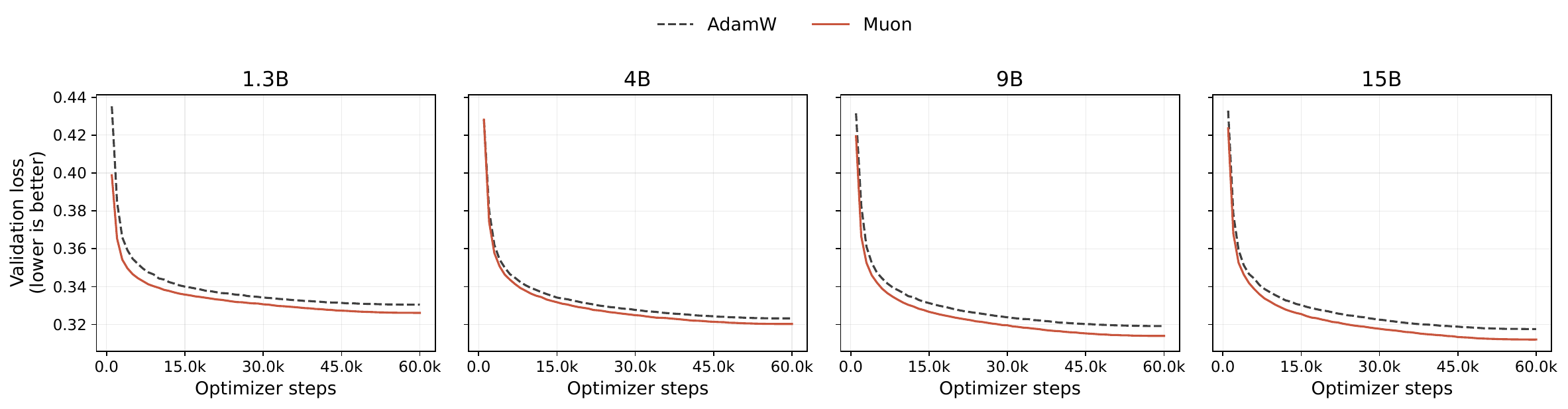}
        \caption{Validation loss versus training progress.}
        \label{fig:val_loss_progress}
    \end{subfigure}

    \vspace{0.0001em}

    \begin{subfigure}[t]{\textwidth}
        \centering
        \includegraphics[
            width=0.7\linewidth,
            trim={0 8pt 0 7pt},
            clip
        ]{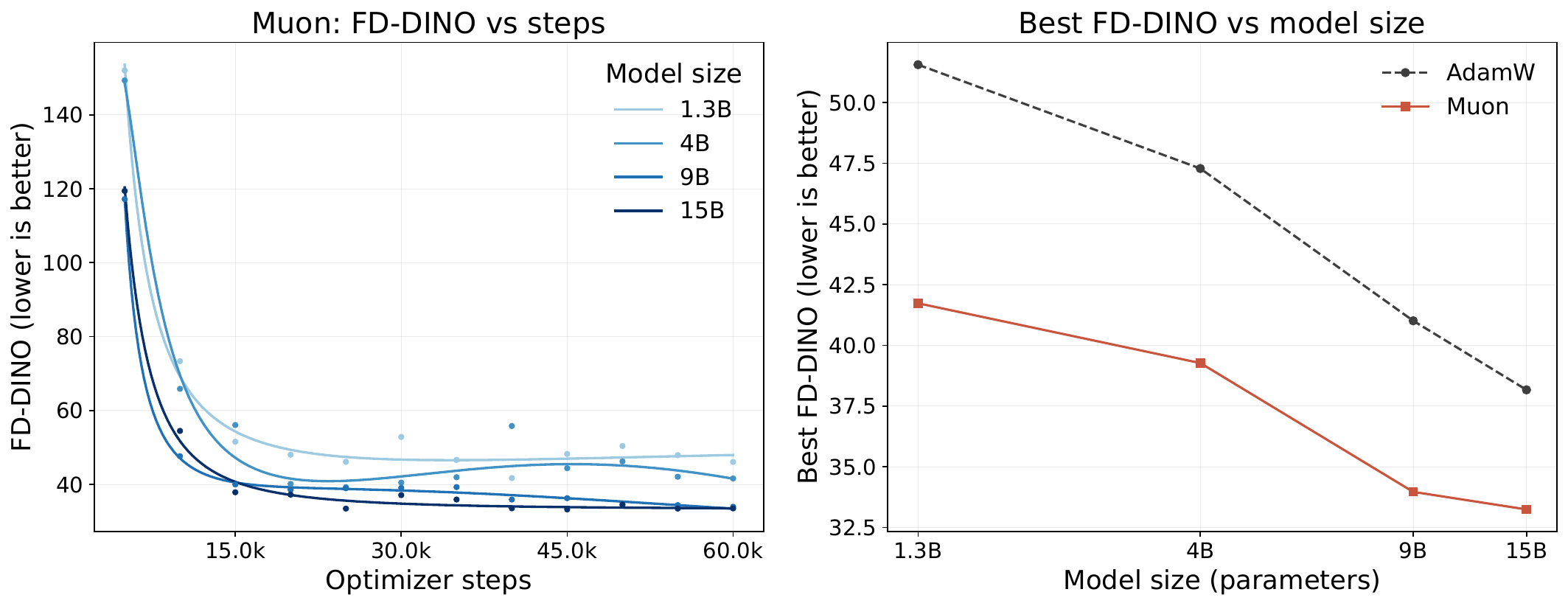}
        \caption{
            FD-DINO throughout training and the best observed FD-DINO
            across model scales.
        }
        \label{fig:fd_dino_scaling}
    \end{subfigure}

    \vspace{0.0001em}

    \begin{subfigure}[t]{\textwidth}
        \centering
        \includegraphics[
            width=\linewidth,
            trim={0 8pt 0 57pt},
            clip
        ]{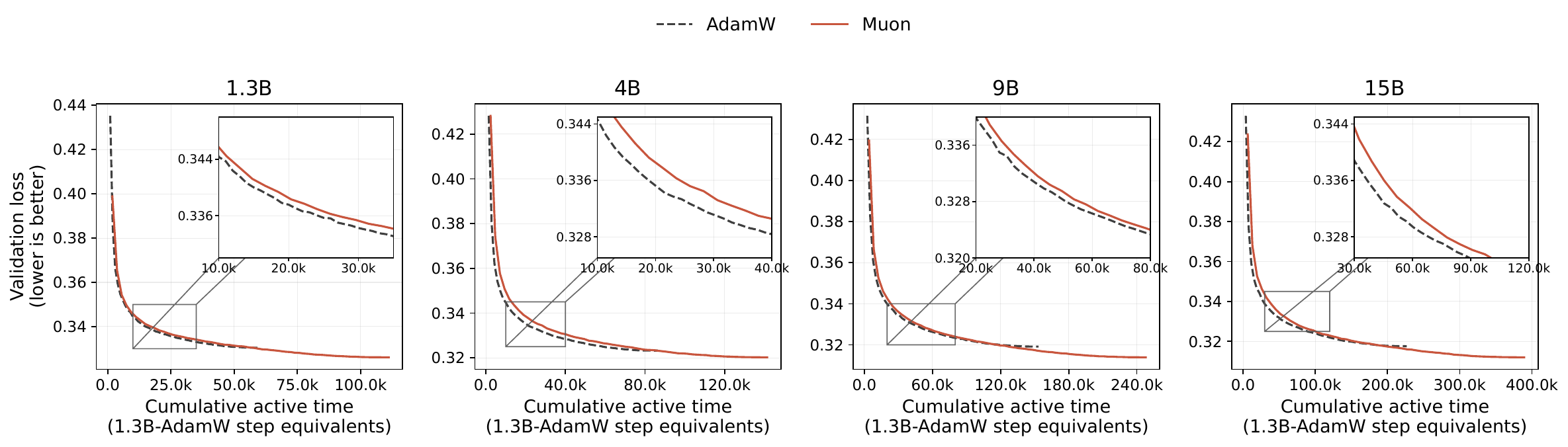}
        \caption{Validation loss versus wall-clock time.}
        \label{fig:val_loss_time}
    \end{subfigure}

    \caption{
        Scaling behavior of Muon for DiTs training.
        Muon consistently improves step efficiency and generation quality, while its higher per-step overhead can offset this advantage in wall-clock time.
    }
    \label{fig:muon_scaling}
\end{figure*}

\subsection{Scaling Muon for Diffusion Transformers}
\label{sec:muon_scaling_results}

We first examine whether Muon retains its optimization advantage as
Diffusion Transformers scale from 1.3B to 15B parameters.
Figure~\ref{fig:muon_scaling}(a) compares the validation loss trajectories
of AdamW and Muon over 60k optimization steps.
Across all four model scales, Muon consistently achieves lower validation
loss than AdamW throughout the main training regime.
The achievable validation loss also decreases as model size increases,
showing that Muon's optimization advantage remains stable when scaling
to substantially larger DiTs.


This advantage also extends to generation quality. Figure~\ref{fig:muon_scaling}(b) shows that, across the 1.3B--15B models, Muon improves the best observed FD-DINO over AdamW by 12.9--19.1\%, with the relative ordering across model scales remaining largely consistent throughout training. Together with the validation loss results, this demonstrates that Muon's optimization advantage persists across scale in both optimization progress and generation quality. Complete checkpoint evaluations are provided in Appendix~\ref{app:checkpoint_evals}.

The step-wise advantage of Muon, however, does not directly translate into
wall-clock time efficiency.
Figure~\ref{fig:muon_scaling}(c) replots the validation loss trajectories
against normalized active training time.
Although Muon ultimately reaches a lower validation loss, its larger
per-step compute and communication cost shifts its trajectory to the right in wall-clock space.
Within the highlighted training regime, AdamW reaches the same intermediate
loss levels earlier than Muon across all 4 model scales.
Thus, the optimization advantage observed per training step is partially
or fully offset by Muon's optimizer overhead under a fixed wall-clock time
budget.
This empirical gap motivates reducing Muon's per-step overhead while preserving
its generation quality advantage.

\begin{figure*}[t]
    \centering
    \includegraphics[
        width=0.68\textwidth,
        trim={0 33pt 0 46pt},
        clip
    ]{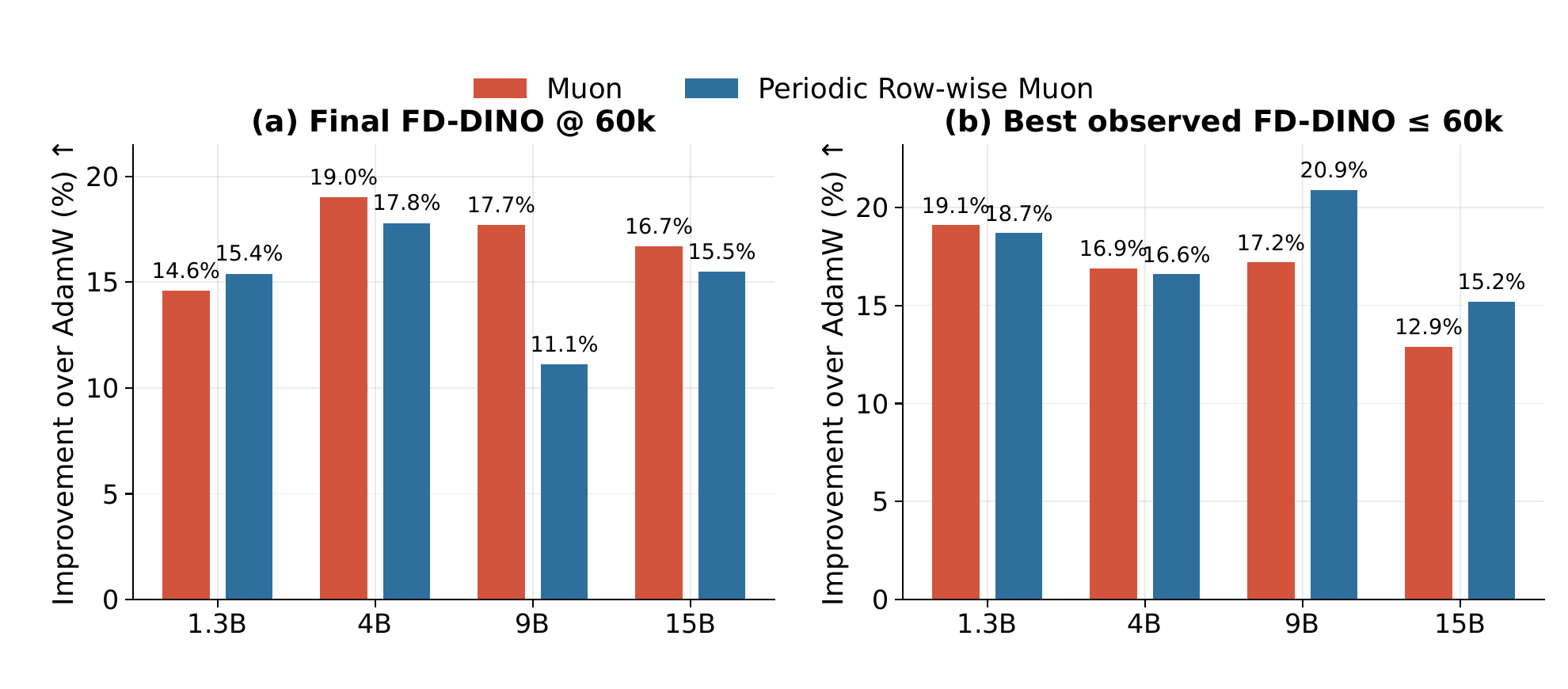}
    \caption{
        Generative quality across model scales.
        We report FD-DINO improvement over AdamW at both the final checkpoint and the best
        checkpoint observed during training.
    }
    \label{fig:fd_summary}
\end{figure*}

\begin{figure*}[t]
    \centering
    \includegraphics[
        width=\textwidth,
        trim={0 8pt 0 13pt},
        clip
    ]{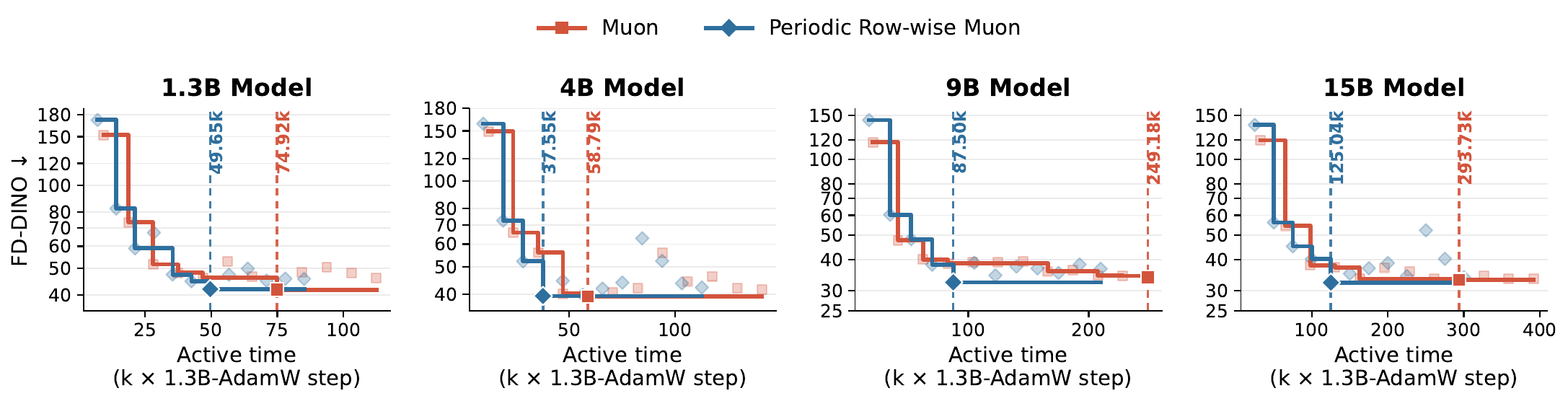}
    \caption{
        Generative quality frontier under optional checkpoint selection.
    }
    \label{fig:checkpoint_frontier}
\end{figure*}

\subsection{Periodic Row-wise Muon Generation Quality and Training Efficiency}
\label{sec:quality_efficiency}

We next evaluate whether Periodic Row-wise Muon preserves vanilla Muon's
generation quality while reducing training compute and communication cost.
Figure~\ref{fig:fd_summary} shows that it improves final-checkpoint FD-DINO
over AdamW by 11.1--17.8\% across scales and slightly outperforms vanilla
Muon at 1.3B.
Its best-observed FD-DINO remains within 0.5\% of vanilla Muon at 1.3B and
4B, while improving it by approximately 4.5\% at 9B and 2.7\% at 15B.
Table~\ref{tab:final-60k-metrics} shows a similarly competitive broader
generation-quality profile: at 9B, Periodic Row-wise Muon increases GenEval2
AM from 51.77 to 57.33 and GM from 11.35 to 15.97, while at 15B its primary
fidelity metrics remain close to vanilla Muon and Coverage and Density
improve.

Given their comparable best generation quality,
Figure~\ref{fig:checkpoint_frontier} compares the best achieved FD-DINO
against normalized active training time.
Periodic Row-wise Muon reaches its respective best result with
33.7\%, 36.1\%, 64.8\%, and 57.4\% less active time at 1.3B, 4B, 9B, and
15B, respectively, with the largest gains at scale.
A complementary comparison with AdamW in active-time efficiency and final
generation quality is provided in
Appendix~\ref{app:adamw_periodic_active_time}.

\begin{table*}[t]
  \centering
  \caption{Final 60k-step comparison.
  Bold denotes the best result among optimizers at the same model size.}
  \label{tab:final-60k-metrics}
  \setlength{\tabcolsep}{0pt}
  \renewcommand{\arraystretch}{1.08}
  \fontsize{6.0}{7.2}\selectfont
  \resizebox{0.9\textwidth}{!}{
\begin{tabular*}{\textwidth}{@{\extracolsep{\fill}}llcccccccccc@{}}
\toprule
\multirow{2}{*}{Model} &
\multirow{2}{*}{Optimizer} &
\multicolumn{3}{c}{Discrepancy} &
\multicolumn{4}{c}{Fidelity \& Diversity} &
\multicolumn{1}{c}{Alignment} &
\multicolumn{2}{c}{Compositionality} \\
\cmidrule(lr){3-5}\cmidrule(lr){6-9}\cmidrule(lr){10-10}\cmidrule(lr){11-12}
& &
FD-DINO$\downarrow$ &
FID$\downarrow$ &
MMD-DINO$\downarrow$ &
Precision$\uparrow$ &
Recall$\uparrow$ &
Coverage$\uparrow$ &
Density$\uparrow$ &
HPSv2$\uparrow$ &
GenEval2 AM$\uparrow$ &
GenEval2 GM$\uparrow$ \\
\midrule
\multirow{3}{*}{1.3B} & AdamW & $53.93$ & $6.71$ & $0.0312$ & $0.9405$ & $0.8874$ & $0.9401$ & $0.9665$ & $20.53$ & $38.66$ & $7.10$ \\
 & Muon & $46.08$ & $6.54$ & $\mathbf{0.0268}$ & $0.9487$ & $\mathbf{0.9071}$ & $\mathbf{0.9492}$ & $\mathbf{0.9918}$ & $\mathbf{20.84}$ & $44.72$ & $9.38$ \\
 & Periodic Row-wise Muon & $\mathbf{45.65}$ & $\mathbf{6.29}$ & $0.0270$ & $\mathbf{0.9517}$ & $0.9065$ & $0.9485$ & $0.9756$ & $20.36$ & $\mathbf{45.89}$ & $\mathbf{9.75}$ \\
\midrule
\multirow{3}{*}{4B}
& AdamW
& $51.39$
& $7.35$
& $0.0271$
& $0.9371$
& $0.8995$
& $0.9388$
& $\mathbf{0.9898}$
& $\mathbf{21.40}$
& $45.70$
& $9.13$ \\

& Muon
& $\mathbf{41.62}$
& $\mathbf{6.08}$
& $0.0261$
& $0.9385$
& $\mathbf{0.9136}$
& $0.9431$
& $0.9639$
& $20.66$
& $46.58$
& $9.06$ \\

& Periodic Row-wise Muon
& $42.25$
& $6.35$
& $\mathbf{0.0218}$
& $\mathbf{0.9413}$
& $0.9057$
& $\mathbf{0.9516}$
& $0.9837$
& $21.34$
& $\mathbf{49.38}$
& $\mathbf{10.58}$ \\
\midrule
\multirow{3}{*}{9B} & AdamW & $41.26$ & $6.19$ & $0.0216$ & $0.9455$ & $0.9045$ & $0.9495$ & $\mathbf{0.9873}$ & $21.39$ & $45.93$ & $9.66$ \\
 & Muon & $\mathbf{33.97}$ & $\mathbf{5.61}$ & $0.0199$ & $\mathbf{0.9516}$ & $\mathbf{0.9269}$ & $\mathbf{0.9531}$ & $0.9601$ & $\mathbf{22.60}$ & $51.77$ & $11.35$ \\
 & Periodic Row-wise Muon & $36.70$ & $6.13$ & $\mathbf{0.0192}$ & $0.9494$ & $0.9134$ & $0.9511$ & $0.9793$ & $22.44$ & $\mathbf{57.33}$ & $\mathbf{15.97}$ \\
\midrule
\multirow{3}{*}{15B} & AdamW & $40.23$ & $6.24$ & $0.0217$ & $0.9448$ & $0.9095$ & $0.9473$ & $0.9638$ & $21.99$ & $47.34$ & $9.18$ \\
 & Muon & $\mathbf{33.51}$ & $\mathbf{5.57}$ & $\mathbf{0.0191}$ & $0.9445$ & $\mathbf{0.9304}$ & $0.9501$ & $0.9367$ & $\mathbf{22.37}$ & $\mathbf{57.93}$ & $\mathbf{14.55}$ \\
 & Periodic Row-wise Muon & $33.99$ & $5.63$ & $0.0197$ & $\mathbf{0.9482}$ & $0.9248$ & $\mathbf{0.9554}$ & $\mathbf{0.9673}$ & $21.35$ & $52.26$ & $11.97$ \\
\bottomrule
\end{tabular*}}
\end{table*}

\subsection{System Efficiency}
\label{sec:system_efficiency}

We next analyze the systems-level source of the wall-clock time improvement in Table~\ref{tab:system_efficiency}.
Across the 4 model scales, Periodic Row-wise Muon reduces end-to-end
step time by 15.7--24.3\% relative to vanilla Muon, while reducing
optimizer time by 46.9--54.3\%.
On the largest 15B model, optimizer time decreases by 54.3\%, resulting in a 23.4\% reduction
in total step time.

The communication measurements exhibit the expected reduction from
periodic spectral refreshes.
With $K=3$, the average number of optimizer specific bucketed all-gathers is reduced
by around 66.6\%.
The corresponding logical communication volume decreases by 66.7\%
at every scale.
The sharded RowNorm path introduces only approximately 0.7 small
all-reduces per step.

Figure~\ref{fig:profiler_15b} shows the execution breakdown for the 15B model.
Periodic Row-wise Muon replaces two of every three NS5 phases and their
full-momentum all-gathers with short RowNorm updates, while overlapping
refresh communication with NS5 computation and norm-statistics communication
with local computation.
This reduces both optimizer computation and exposed communication.

\begin{table}[t]
\centering
\caption{System-efficiency comparison on 32 H100 nodes.}
\label{tab:system_efficiency}

\small
\setlength{\tabcolsep}{2.2pt}
\renewcommand{\arraystretch}{1.10}

\resizebox{0.7\textwidth}{!}{%
\begin{tabular}{
@{}
l@{\hspace{8pt}}l
r@{\hspace{2pt}}G
r@{\hspace{2pt}}G
c
c
c
r@{\hspace{2pt}}G
@{}
}
\toprule

\multirow{2}{*}{Model}
&
\multirow{2}{*}{Method}
&
\multicolumn{5}{c}{Computation}
&
\multicolumn{4}{c}{Communication}
\\

\cmidrule(lr){3-7}
\cmidrule(l){8-11}

&
&
\multicolumn{2}{c}{Step time $\downarrow$}
&
\multicolumn{2}{c}{Opt. time $\downarrow$}
&
\multicolumn{1}{c}{Opt. share $\downarrow$}
&
AG/step $\downarrow$
&
AR/step $\downarrow$
&
\multicolumn{2}{c}{
  \makecell[c]{
    Comm. volume $\downarrow$\\[-1pt]
    \scriptsize GiB/rank/step
  }
}
\\

\midrule

\multirow{2}{*}{1.3B}
& Muon
& 1.873 & {}
& 0.835 & {}
& 44.6\%
& 10.0
& 0.0
& 2.38 & {}
\\

& Periodic Row-wise Muon
& 1.419 & {$(-24.3\%)$}
& 0.443 & {$(-46.9\%)$}
& 31.3\%
& 3.3
& 0.7
& 0.79 & {$(-66.7\%)$}
\\

\midrule

\multirow{2}{*}{4B}
& Muon
& 2.352 & {}
& 1.009 & {}
& 42.9\%
& 32.0
& 0.0
& 7.58 & {}
\\

& Periodic Row-wise Muon
& 1.878 & {$(-20.2\%)$}
& 0.536 & {$(-46.9\%)$}
& 28.6\%
& 10.7
& 0.7
& 2.53 & {$(-66.7\%)$}
\\

\midrule

\multirow{2}{*}{9B}
& Muon
& 4.153 & {}
& 1.689 & {}
& 40.7\%
& 80.0
& 0.0
& 17.30 & {}
\\

& Periodic Row-wise Muon
& 3.500 & {$(-15.7\%)$}
& 0.868 & {$(-48.6\%)$}
& 24.8\%
& 26.7
& 0.7
& 5.77 & {$(-66.7\%)$}
\\

\midrule

\multirow{2}{*}{15B}
& Muon
& 6.527 & {}
& 2.884 & {}
& 44.2\%
& 119.0
& 0.0
& 28.62 & {}
\\

& Periodic Row-wise Muon
& 5.002 & {$(-23.4\%)$}
& 1.317 & {$(-54.3\%)$}
& 26.3\%
& 39.7
& 0.7
& 9.54 & {$(-66.7\%)$}
\\


\bottomrule
\end{tabular}%
}
\end{table}

\begin{figure*}[t]
    \centering
    \includegraphics[
        width=0.72\textwidth,
        trim={70pt 25pt 0 15pt},
        clip
    ]
        {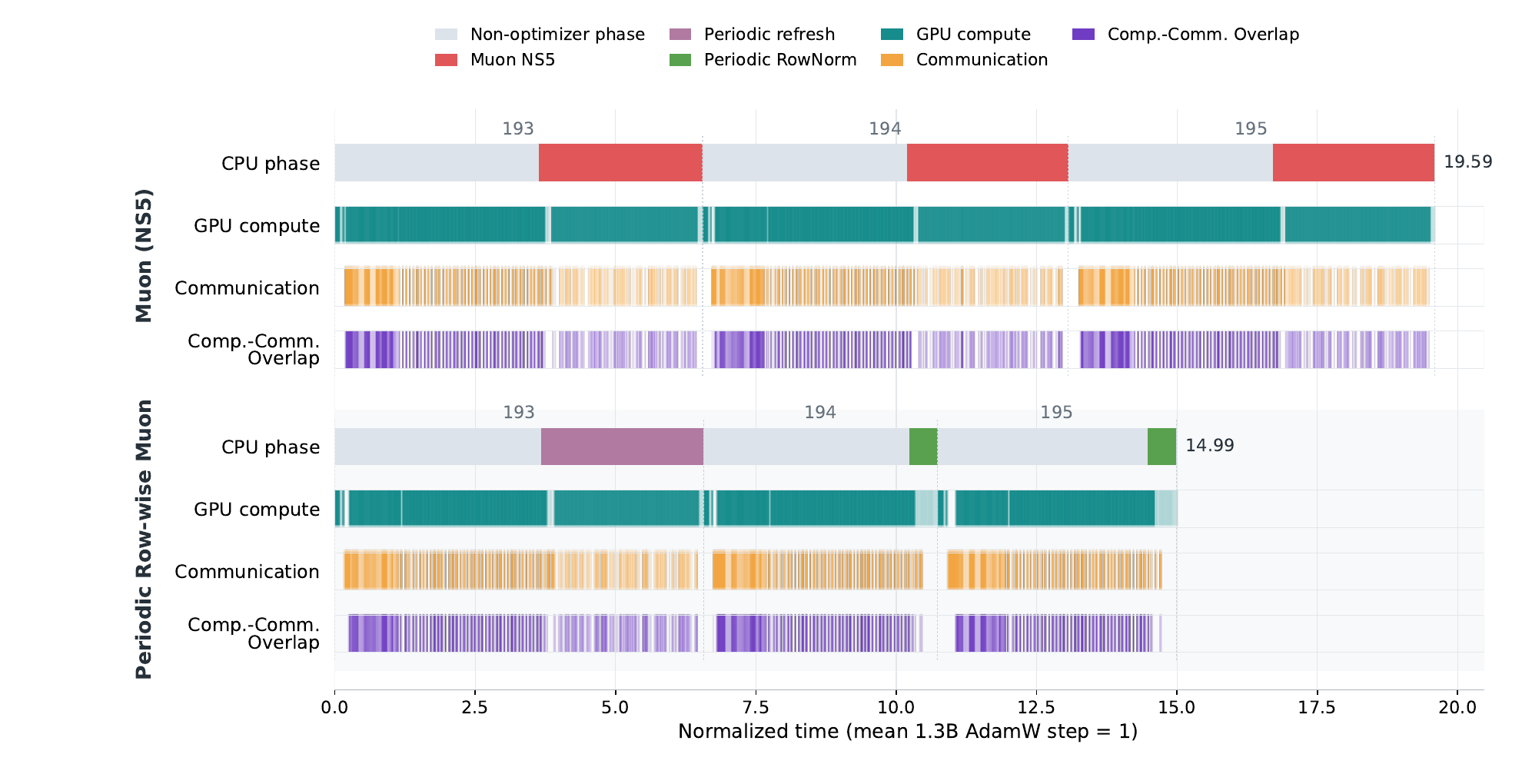}
    \caption{
        Simplified profiler traces for the 15B model.
    }
    \label{fig:profiler_15b}
\end{figure*}

\subsection{Ablation Studies}
\label{sec:ablations}

\textbf{Algorithmic ablation.}
Table~\ref{tab:ablations} (a) isolates the roles of periodic spectral
refresh, RowNorm geometry, and branch-specific scale calibration.
We compare vanilla Muon, RowNorm at every step, Periodic Row-wise Muon
with $\gamma=1$, a scalar-controlled baseline that applies $\gamma$ to both NS5 and RowNorm, and the complete
Periodic Row-wise Muon with $K=3$ and $\gamma=0.15$ on 1.3B scale.
Overall, the complete method provides the best generation
quality.

\begin{table}[t]
    \centering
    \caption{
        Ablation studies.
        (a) Algorithm-level ablation of Periodic Row-wise Muon on the
        1.3B model.
        (b) System-level ablation of the distributed implementation
        on the 15B model.
    }
    \label{tab:ablations}

    \vspace{-2pt}

    \resizebox{0.78\linewidth}{!}{%
    \begin{tabular}{@{}lcccc@{}}
        \toprule
        \textbf{(a)} Variant
        & NS5 Schedule
        & Off-refresh Update
        & $\gamma$
        & Best FD-DINO $\downarrow$ \\
        \midrule

        Vanilla Muon
        & Every step
        & --
        & --
        & 41.733 \\

        RowNorm every step
        & None
        & RowNorm
        & 0.15
        & 51.322 \\

        Periodic Row-wise Muon ($\gamma=1$)
        & $K=3$
        & RowNorm
        & 1.0
        & 44.294 \\

        Scalar-controlled ($\gamma$ for both NS5 and RowNorm)
        & $K=3$
        & Scalar-controlled
        & 0.15
        & 44.713 \\

        Periodic Row-wise Muon ($\gamma=0.15$)
        & $K=3$
        & RowNorm
        & 0.15
        & 41.913 \\

        \bottomrule
    \end{tabular}%
    }

    \par\vspace{5pt}

    \resizebox{0.82\linewidth}{!}{%
    \begin{tabular}{@{}lcccccc@{}}
        \toprule
        \textbf{(b)} Variant
        & Sharded
        & Bucketed
        & Comm.--Comp.
        & Step Time $\downarrow$
        & Opt. Time $\downarrow$
        & Comm. Volume $\downarrow$ \\
        & RowNorm
        & All-gather
        & Overlap
        & & & \\
        \midrule

        Naive Periodic $(K=3,\gamma=0.15)$
        & $\times$
        & $\times$
        & $\times$
        & 5.685
        & 2.046
        & 28.62 \\

        + Sharded RowNorm
        & $\checkmark$
        & $\times$
        & $\times$
        & 5.196
        & 1.498
        & $\mathbf{9.54}$ \\

        + Bucketed All-gather
        & $\checkmark$
        & $\checkmark$
        & $\times$
        & 5.028
        & 1.349
        & $\mathbf{9.54}$ \\

        Full System
        & $\checkmark$
        & $\checkmark$
        & $\checkmark$
        & $\mathbf{5.002}$
        & $\mathbf{1.317}$
        & $\mathbf{9.54}$ \\

        \bottomrule
    \end{tabular}%
    }
\end{table}


\textbf{System ablation.}
Table~\ref{tab:ablations}(b) incrementally adds the system optimizations
to a 15B naive periodic baseline, which uses the same update rule but
materializes full momentum every step without bucketing or overlap.
Sharded RowNorm reduces communication volume by \(66.7\%\), while bucketed
all-gather and communication--computation overlap further reduce exposed
refresh latency.
Together, these optimizations reduce optimizer and end-to-end step time by
\(35.6\%\) and \(12.0\%\), respectively.








\section{Related Work}

Unlike AdamW's element-wise adaptive updates
\citep{loshchilov2017decoupled}, matrix-aware optimizers exploit the matrix
structure of model parameters. Shampoo and SOAP construct structured
preconditioners \citep{gupta2018shampoo,vyas2025soap}, whereas Muon applies a
finite Newton--Schulz transformation to momentum matrices, producing updates
with global spectral structure \citep{jordan2024muon}. Muon has subsequently
been shown to scale to large language models and improve the compute--quality
trade-off over AdamW \citep{liu2025muon,shah2025practical}. In parallel,
prior work has established predictable DiT scaling across model size, data,
compute, and training hyperparameters
\citep{esser2024scaling,liang2026scaling,yin2025towards}, while recent studies
evaluate matrix-aware optimization in diffusion training and adapt momentum
orthogonalization to diffusion-specific parameter structures
\citep{schaipp2025optimization,chen2026cmuon}. Together, these directions
motivate evaluating optimizer scalability through optimization progress,
algorithmic compute, generative quality, and realized distributed efficiency.

Recent work reduces 
the computation required by Muon's spectral transformation, 
which is particularly
important for large DiTs. One direction accelerates orthogonalization through
GPU-friendly polynomial iterations, Gram-matrix formulations, or optimized
Newton--Schulz polynomials
\citep{amsel2026polar,GramNewtonSchulz,grishina2025accelerating}; another
reduces its frequency, scope, or input size through block-periodic updates,
temporal reuse, alternating spectral and sign-based updates, tiled
transformations, or row/column subsampling
\citep{khaled2026muonbp,dev2026cachemuon,bolatov2026lionmuon,
tang2026hierarchical,ahn2025dion2}. Other variants post-process every step
NS orthogonalization with row- or neuron-wise normalization and optional
neuron-wise second-moment statistics
\citep{li2025normuon,zhang2026muon+}. Distributed orthonormalized optimizers
and specialized Muon runtimes further reduce communication, redundant
computation, and exposed optimizer latency under model sharding
\citep{ahn2025dion,chen2026dmuon}. Our approach is complementary to faster
spectral kernels, it reduces spectral refresh frequency and directly applies
RowNorm to the current momentum as a low compute and communication cost dense substitute between
refreshes. RowNorm also admits recently studied symmetry properties
\citep{lau2026symmetry}, and its locality allows us to avoid full-momentum
materialization on non-refresh steps in large-scale distributed training.

\section{Conclusion and Limitations}

We show that Muon's optimization and generative quality advantages over AdamW persist as DiTs scale from 1.3B to 15B parameters, 
but its every-step NS5 computation and full-momentum communication substantially increase per-step execution time in distributed training.
Periodic Row-wise Muon addresses this bottleneck by combining periodic spectral refreshes with low compute and communication cost RowNorm updates and a sharded distributed implementation. Across all scales, it preserves generation quality broadly comparable to vanilla Muon while reducing optimizer time by 46.9--54.3\%.
However, our evaluation is limited to one DiT family, dataset, resolution, and 32-node H100 configuration, and does not cover other training regimes. We also use fixed global \(K\) and \(\gamma\), leaving layer-wise and adaptive schedules unexplored. Finally, refresh steps still retain full momentum communication and materialization, while the realized speedup may vary with hardware topology and parallelism strategy. We will leave these limitations as future works.

\bibliography{paper}
\bibliographystyle{iclr2027_conference}



\clearpage
\appendix

\section*{Appendix}

\providecommand{\RN}{\operatorname{\mathcal{R}}}
\providecommand{\Polar}{\operatorname{polar}}
\providecommand{\NSfive}{\operatorname{NS5}}
\providecommand{\E}{\mathbb{E}}
\providecommand{\R}{\mathbb{R}}
\providecommand{\Frob}{\mathrm{F}}

\section{Mathematical Analysis}
\label{app:mathematical-analysis}

This appendix gives the complete statements and proofs supporting the
geometric interpretation and the optimization-feasibility discussion in the
main text.  The results deliberately distinguish the exact
polar factor from the finite Newton--Schulz map used in the implementation.
They also distinguish a conditional descent guarantee from an unconditional
convergence theorem for the proposed optimizer.

\subsection{Notation and update convention}
\label{app:math-notation}

For matrices of the same shape, let
\begin{equation}
\label{eq:frobenius-inner-product}
\langle A,B\rangle \coloneqq \operatorname{tr}(A^\top B),
\quad
\|A\|_{\mathrm F}^{2} \coloneqq \langle A,A\rangle,
\end{equation}
and let \(\|A\|_{2}\) denote the spectral norm.  For
\(X\in\R^{r\times c}\), define the row-wise maximum norm
\begin{equation}
    \|X\|_{2,\infty}:=\max_{1\leq i\leq r}\|X_{i:}\|_{2}.
\end{equation}
The practical RowNorm map is
\begin{equation}
    \RN_{\varepsilon}(X)_{i:}
    :=
    \frac{X_{i:}}{\max\{\|X_{i:}\|_{2},\varepsilon\}},
    \qquad \varepsilon>0.
    \label{eq:app-rownorm-eps}
\end{equation}
When every row is nonzero, \(\RN(X)\) denotes the unregularized map
obtained by setting the denominator to \(\|X_{i:}\|_{2}\).

For the optimization result, \(\theta_t\) denotes the vector containing all
trainable parameters, and we write the implemented update as
\begin{equation}
    \theta_{t+1}=\theta_t-\eta_t H_t.
    \label{eq:app-effective-update}
\end{equation}
Here \(H_t\) is the complete effective direction. It includes the
selected NS5 or RowNorm matrix blocks, Muon's shape factors, the RowNorm
multiplier, all AdamW-updated blocks, any parameter-group learning rate
ratios relative to \(\eta_t\), and decoupled weight decay.  This convention is
important because the RowNorm multiplier scales only the RowNorm blocks, not
the complete optimizer direction.

\subsection{RowNorm as a row-wise variational direction}
\label{app:rownorm-variational}



\begin{lemma}[Row-wise variational characterization]
\label{lem:app-rownorm-variational}
Let \(X\in\mathbb{R}^{r\times c}\), and suppose that every row of
\(X\) is nonzero. Define
\[
\|U\|_{2,\infty}:=\max_{1\le i\le r}\|U_{i,:}\|_2.
\]
Then
\begin{equation}
R(X)
=
\operatorname*{arg\,max}_{U\in\mathbb{R}^{r\times c}}
\left\{
\langle X,U\rangle:
\|U\|_{2,\infty}\le 1
\right\},
\end{equation}
where the maximizer is unique.  
\end{lemma}

\begin{proof}
The constraint \(\|U\|_{2,\infty}\leq 1\) is equivalent to
\(\|U_{i:}\|_2\leq 1\) for every row.  Hence, by Cauchy--Schwarz,
\begin{equation}
    \langle X,U\rangle
    =\sum_{i=1}^{r}\langle X_{i:},U_{i:}\rangle
    \leq
    \sum_{i=1}^{r}\|X_{i:}\|_2\|U_{i:}\|_2
    \leq
    \sum_{i=1}^{r}\|X_{i:}\|_2.
\end{equation}
Equality is attained by
\(U_{i:}=X_{i:}/\|X_{i:}\|_2\) for every \(i\), which is exactly
\(U=\RN(X)\).  Because every \(X_{i:}\) is nonzero, equality in the two
inequalities requires this choice row by row, proving uniqueness.
\end{proof}

Lemma~\ref{lem:app-rownorm-variational} characterizes RowNorm in its own
geometry.  In particular, it neither treats RowNorm as an approximation to
NS5 nor asserts that the two updates follow the same optimization trajectory.

\subsection{Polar and Row-wise Variational Geometries}
\label{app:distinct-geometries}

We use the orientation convention from Section \ref{sec:complementary-geometries}: a matrix is
transposed when necessary so that \(X\in\R^{r\times c}\) with \(r\leq c\).
If \(X\) has full row rank, its row-polar factor is
\begin{equation}
    \Polar(X):=(XX^{\top})^{-1/2}X,
    \qquad
    \Polar(X)\Polar(X)^{\top}=I_r.
    \label{eq:app-polar-definition}
\end{equation}

\begin{proposition}[Spectral variational characterization]
\label{prop:app-polar-variational}
Let \(X\in\R^{r\times c}\), where \(r\leq c\), and suppose that
\(\operatorname{rank}(X)=r\).  Then
\begin{equation}
    \Polar(X)
    =
    \underset{U\in\R^{r\times c}}{\arg\max}\;
    \langle X,U\rangle
    \quad\text{subject to}\quad
    \|U\|_2\leq 1.
    \label{eq:app-polar-variational}
\end{equation}
The maximizer is unique, and the optimal value is the nuclear norm
\(\|X\|_*\).
\end{proposition}

\begin{proof}
Let \(X=A\Sigma B^{\top}\) be a thin singular value decomposition, where
\(A\in\R^{r\times r}\) is orthogonal,
\(B\in\R^{c\times r}\) has orthonormal columns, and
\(\Sigma=\operatorname{diag}(\sigma_1,\ldots,\sigma_r)\) with
\(\sigma_i>0\).  Extend \(B\) to a square orthogonal matrix
\(V=[B\;B_{\perp}]\in\R^{c\times c}\).  For any feasible \(U\), set
\(Z=A^{\top}UV\).  Orthogonal invariance gives \(\|Z\|_2\leq 1\), and
\begin{equation}
    \langle X,U\rangle
    =\operatorname{tr}(\Sigma A^{\top}UB)
    =\sum_{i=1}^{r}\sigma_i Z_{ii}
    \leq\sum_{i=1}^{r}\sigma_i
    =\|X\|_*.
\end{equation}
The upper bound is attained by \(U=AB^{\top}=\Polar(X)\).  Because every
\(\sigma_i\) is positive, equality requires \(Z_{ii}=1\) for every
\(i\leq r\).  A contraction whose first \(r\) diagonal entries are all one
must have \(Z=[I_r\;0]\); otherwise at least one row would have Euclidean norm
greater than one.  Thus \(U=AB^{\top}\) is the unique maximizer.
\end{proof}

The two feasible sets obey
\begin{equation}
    \{U:\|U\|_2\leq 1\}
    \subseteq
    \{U:\|U\|_{2,\infty}\leq 1\}
    \subseteq
    \{U:\|U\|_2\leq\sqrt{r}\}.
    \label{eq:app-feasible-set-inclusions}
\end{equation}
Indeed, every row norm is at most the spectral norm, while
\(\|U\|_2\leq\|U\|_{\mathrm F}\leq\sqrt{r}\|U\|_{2,\infty}\).
Thus the polar direction solves a globally coupled spectral-norm problem,
whereas RowNorm solves a product of independent row-wise problems.  Finite
NS5 is the implemented spectral map motivated by the former geometry; the
proposition does not identify finite NS5 with the exact polar factor.

\subsection{Conditional stability of the ideal polar direction}
\label{app:polar-stability}

\begin{proposition}[Perturbation of the row-polar factor]
\label{prop:app-polar-perturbation}
Let \(A,B\in\R^{r\times c}\), where \(r\leq c\), and suppose that both
matrices have full row rank.  Then
\begin{equation}
    \|\Polar(A)-\Polar(B)\|_{\mathrm F}
    \leq
    \frac{2\|A-B\|_{\mathrm F}}
    {\sigma_{\min}(A)+\sigma_{\min}(B)}.
    \label{eq:app-polar-perturbation}
\end{equation}
Consequently, on any region satisfying
\(\sigma_{\min}(X)\geq\sigma_0>0\), the polar map is
\(1/\sigma_0\)-Lipschitz in Frobenius norm.
\end{proposition}

\begin{proof}
Write the row-polar decompositions as
\begin{equation}
    A=H_A P_A,
    \qquad
    B=H_B P_B,
\end{equation}
where \(H_A=(AA^{\top})^{1/2}\),
\(H_B=(BB^{\top})^{1/2}\),
\(P_A=\Polar(A)\), and \(P_B=\Polar(B)\).  Both
\(P_AP_A^{\top}=I_r\) and \(P_BP_B^{\top}=I_r\).  Let
\(C=P_AP_B^{\top}\), \(\operatorname{sym}(C):=(C+C^\top)/2\).  Since \(\|C\|_2\leq 1\),
\(I_r-\operatorname{sym}(C)\) is positive semidefinite.  Direct expansion
gives
\begin{align}
    \langle A-B,P_A-P_B\rangle
    &=\operatorname{tr}\!\left(H_A(I_r-C^{\top})\right)
      +\operatorname{tr}\!\left(H_B(I_r-C)\right) 
      \notag \\
    &\geq
      \bigl(\sigma_{\min}(A)+\sigma_{\min}(B)\bigr)
      \bigl(r-\operatorname{tr}(C)\bigr).
\end{align}
The skew-symmetric parts vanish inside the traces because \(H_A\) and
\(H_B\) are symmetric.  Moreover,
\begin{equation}
    r-\operatorname{tr}(C)
    =\frac{1}{2}\|P_A-P_B\|_{\mathrm F}^{2}.
\end{equation}
Combining this identity with Cauchy--Schwarz yields
\begin{equation}
    \|A-B\|_{\mathrm F}\|P_A-P_B\|_{\mathrm F}
    \geq
    \frac{\sigma_{\min}(A)+\sigma_{\min}(B)}{2}
    \|P_A-P_B\|_{\mathrm F}^{2}.
\end{equation}
If \(P_A=P_B\), the result is immediate; otherwise division by
\(\|P_A-P_B\|_{\mathrm F}\) proves~\eqref{eq:app-polar-perturbation}.
\end{proof}

Proposition~\ref{prop:app-polar-perturbation} provides a conditional
motivation for periodic spectral correction: moderate momentum drift
implies moderate drift of the ideal polar direction only when the
relevant matrices remain away from rank degeneracy. The proposition
concerns the exact polar factor and does not establish the same
perturbation bound for the finite NS5 map used in the implementation.
We therefore use it as geometric motivation rather than as a guarantee
that finite NS5 remains unchanged between refreshes.

\subsection{Conditional finite-horizon descent}
\label{app:conditional-descent}

Let \(\mathcal{F}_t\) denote the information available before the stochastic
direction at step \(t\) is realized.  The following result applies directly
to the complete effective direction in~\eqref{eq:app-effective-update}; it
does not require the iterates to track those of vanilla Muon.

\begin{theorem}[Finite-horizon bound under conditional alignment
]
\label{thm:app-conditional-descent}
Suppose that \(f:\R^d\to\R\) is \(L\)-smooth and bounded below by
\(f_{\inf}\).  Assume that, for each \(t\), there are deterministic constants
\(a_t>0\), \(b_t\geq 0\), and \(v_t\geq 0\) such that
\begin{align}
    \E\!\left[
        \langle\nabla f(\theta_t),H_t\rangle
        \mid\mathcal{F}_t
    \right]
    &\geq a_t\|\nabla f(\theta_t)\|_2^2,
    \label{eq:app-alignment-assumption}\\
    \E\!\left[
        \|H_t\|_2^2
        \mid\mathcal{F}_t
    \right]
    &\leq b_t\|\nabla f(\theta_t)\|_2^2+v_t.
    \label{eq:app-moment-assumption}
\end{align}
If
\begin{equation}
    d_t:=a_t-\frac{L\eta_t b_t}{2}>0
    \qquad\text{for }t=0,\ldots,T-1,
    \label{eq:app-positive-descent-coefficient}
\end{equation}
then
\begin{equation}
    \sum_{t=0}^{T-1}
    \eta_t d_t\,
    \E\|\nabla f(\theta_t)\|_2^2
    \leq
    f(\theta_0)-f_{\inf}
    +\frac{L}{2}\sum_{t=0}^{T-1}\eta_t^2v_t.
    \label{eq:app-finite-horizon-bound}
\end{equation}
This statement holds for any deterministic switching schedule, including a
periodic schedule with one NS5 step followed by \(K-1\) RowNorm steps.
\end{theorem}

\begin{proof}
By \(L\)-smoothness and~\eqref{eq:app-effective-update},
\begin{equation}
    f(\theta_{t+1})
    \leq
    f(\theta_t)
    -\eta_t\langle\nabla f(\theta_t),H_t\rangle
    +\frac{L\eta_t^2}{2}\|H_t\|_2^2.
\end{equation}
Taking conditional expectation and applying
\eqref{eq:app-alignment-assumption} and \eqref{eq:app-moment-assumption} gives
\begin{equation}
    \E[f(\theta_{t+1})\mid\mathcal{F}_t]
    \leq
    f(\theta_t)
    -\eta_t d_t\|\nabla f(\theta_t)\|_2^2
    +\frac{L\eta_t^2}{2}v_t.
\end{equation}
Taking total expectation, summing over \(t=0,\ldots,T-1\), and using
\(f(\theta_T)\geq f_{\inf}\) proves
\eqref{eq:app-finite-horizon-bound}.
\end{proof}

For example, if \(\eta_t=\eta\), \(a_t\geq a>0\),
\(0\leq b_t\leq b\) for some \(b>0\), \(v_t\leq v\), and \(0<\eta<2a/(Lb)\), then
\begin{equation}
    \frac{1}{T}\sum_{t=0}^{T-1}
    \E\|\nabla f(\theta_t)\|_2^2
    \leq
    \frac{f(\theta_0)-f_{\inf}}
    {T\eta\left(a-L\eta b/2\right)}
    +
    \frac{L\eta v}{2\left(a-L\eta b/2\right)}.
    \label{eq:app-constant-step-bound}
\end{equation}
This is a finite-horizon stationary-point bound with a noise-dependent
residual, not a claim of asymptotic convergence for the finite cosine
schedule used in our experiments.

The theorem is explicitly conditional.  RowNorm is positively aligned with
its own input momentum whenever the momentum has a nonzero row:
\begin{equation}
    \langle M,\RN(M)\rangle
    =\sum_{i=1}^{r}\|M_{i:}\|_2>0.
\end{equation}
However, this identity alone does not imply positive alignment with the true
objective gradient, because in general
\(\langle\nabla f(\theta_t),\RN(M_t)\rangle\) and
\(\langle M_t,\RN(M_t)\rangle\) are different quantities.  We therefore use Theorem~\ref{thm:app-conditional-descent} only as a
sufficient-condition result and do not claim that its alignment
assumption is verified throughout training.

\subsection{The blockwise role of the RowNorm multiplier}
\label{app:blockwise-gamma}

Conditioned on a fixed optimization history \(\mathcal F_t\) and the
current iterate \(\theta_t\), decompose the complete effective direction
on an off-refresh step as
\begin{equation}
    H_t(\gamma)=H_t^{(0)}+\gamma H_t^{(\mathrm{RN})},
    \label{eq:app-gamma-decomposition}
\end{equation}
where \(H_t^{(\mathrm{RN})}\) contains the unscaled RowNorm directions in the
blocks to which the method is applied, and \(H_t^{(0)}\) contains all
remaining contributions.  Define the conditional scalar quantities
\begin{align}
    A_{0,t}&:=\E[\langle\nabla f(\theta_t),H_t^{(0)}\rangle
                    \mid\mathcal{F}_t],\\
    A_{\mathrm{RN},t}&:=\E[\langle\nabla f(\theta_t),H_t^{(\mathrm{RN})}\rangle
                    \mid\mathcal{F}_t],\\
    B_{0,t}&:=\E[\|H_t^{(0)}\|_2^2\mid\mathcal{F}_t],\\
    B_{\mathrm{RN},t}&:=\E[\|H_t^{(\mathrm{RN})}\|_2^2\mid\mathcal{F}_t],\\
    B_{\times,t}&:=\E[\langle H_t^{(0)},H_t^{(\mathrm{RN})}\rangle
                    \mid\mathcal{F}_t].
\end{align}
Then the two quantities entering the smoothness bound are exactly
\begin{align}
    \E[\langle\nabla f(\theta_t),H_t(\gamma)\rangle
        \mid\mathcal{F}_t]
    &=A_{0,t}+\gamma A_{\mathrm{RN},t},
    \label{eq:app-gamma-alignment}\\
    \E[\|H_t(\gamma)\|_2^2\mid\mathcal{F}_t]
    &=B_{0,t}+2\gamma B_{\times,t}+\gamma^2B_{\mathrm{RN},t}.
    \label{eq:app-gamma-second-moment}
\end{align}
Consequently, the conditional one-step upper bound is
\begin{align}
    \E[f(\theta_{t+1})\mid\mathcal{F}_t]
    \leq f(\theta_t)
    &-\eta_t\bigl(A_{0,t}+\gamma A_{\mathrm{RN},t}\bigr)\nonumber\\
    &+\frac{L\eta_t^2}{2}
      \bigl(B_{0,t}+2\gamma B_{\times,t}
      +\gamma^2B_{\mathrm{RN},t}\bigr).
    \label{eq:app-gamma-descent}
\end{align}
Thus \(\gamma\) is admissible whenever the measured or assumed bounds make
the net descent term in~\eqref{eq:app-gamma-descent} positive. Equation \ref{eq:app-gamma-alignment} and \ref{eq:app-gamma-second-moment} also show
why scaling the RowNorm blocks does \emph{not} multiply the alignment and
second-moment constants of the complete direction by \(\gamma\) and
\(\gamma^2\), respectively.  The theory motivates branch-specific amplitude
calibration but does not determine the numerical value used in the
experiments. That value is selected by the development-set ablation in Appendix \ref{app:k_gamma_selection}.

\section{Complexity and Distributed Execution}
\label{app:complexity-systems}

This appendix refines the asymptotic discussion in the main text with an
implementation matched arithmetic model, a correctness proof for sharded
RowNorm, and explicit per-rank communication formulas.

\subsection{Per-matrix arithmetic}
\label{app:per-matrix-arithmetic}

Consider an oriented matrix \(X\in\R^{r\times c}\) with \(r\leq c\).  We
count one multiplication followed by one addition as two floating-point
operations.  The implemented quintic Newton--Schulz step can be written as
\begin{align}
    G_j &= X_jX_j^{\top},\\
    B_j &= b_jG_j+c_jG_j^2,\\
    X_{j+1} &= a_jX_j+B_jX_j,
    \qquad j=0,\ldots,J_{\mathrm{NS}}-1.
    \label{eq:app-ns-recurrence}
\end{align}
The three GEMMs in one iteration need, to leading order, \(2r^2c\), 
    \(2r^3\), 
    \(2r^2c\)
FLOPs, respectively.  Therefore
\begin{equation}
    C_{\mathrm{NS}J}(r,c)
    =J_{\mathrm{NS}}\bigl(4r^2c+2r^3\bigr)
      +O\!\left(J_{\mathrm{NS}}(r^2+rc)\right).
    \label{eq:app-ns-flops}
\end{equation}
For NS5,
\begin{equation}
    C_{\mathrm{NS5}}(r,c)
    =20r^2c+10r^3+O(r^2+rc).
    \label{eq:app-ns5-flops}
\end{equation}
The lower-order term includes scalar matrix combinations and the input
normalization.  Equation~\ref{eq:app-ns5-flops} is a FLOP model for the
recurrence in~\eqref{eq:app-ns-recurrence}; measured kernel time can differ
because GEMM efficiency depends on shape, dtype, and hardware.

RowNorm performs one sum-of-squares reduction per row, one clamped inverse
norm per row, and one rescaling per element.  Its 
arithmetic complexity
is
\begin{equation}
    C_{\mathrm{RN}}(r,c)=\Theta(rc).
    \label{eq:app-rownorm-flops}
\end{equation}
We do not assign a hardware-independent exact FLOP count to square root,
reciprocal, or fused reduction operations.  The relevant separation is
therefore
\begin{equation}
    C_{\mathrm{NS5}}(r,c)=\Theta(r^2c),
    \qquad
    C_{\mathrm{RN}}(r,c)=\Theta(rc).
\end{equation}

Let \(\mathcal{I}\) be the set of matrices governed by the periodic rule.
The period-averaged matrix-processing computational cost is
\begin{equation}
    \overline C_{\mathrm{mat}}(K)
    =\frac{1}{K}\sum_{i\in\mathcal{I}}C_{\mathrm{NS5}}(r_i,c_i)
    +\frac{K-1}{K}\sum_{i\in\mathcal{I}}C_{\mathrm{RN}}(r_i,c_i).
    \label{eq:app-periodic-compute}
\end{equation}
If \(C_{\mathrm{fixed}}\) denotes momentum maintenance, non-Muon parameter
updates, and all other optimizer work performed every step, then
\begin{equation}
    \overline C_{\mathrm{opt}}(K)
    =C_{\mathrm{fixed}}+\overline C_{\mathrm{mat}}(K).
    \label{eq:app-total-optimizer-compute}
\end{equation}
The periodic schedule removes exactly the fraction \(1-1/K\) of spectral
refresh events, but it does not remove the same fraction of total optimizer
FLOPs or wall-clock time because the terms in
\eqref{eq:app-total-optimizer-compute} remain.

\subsection{Correctness of sharded RowNorm}
\label{app:sharded-rownorm-correctness}

Suppose that the original matrix \(M\in\R^{m\times n}\) is sharded along its
first dimension across \(p\) ranks.  Rank \(q\) owns the row-index set
\(S_q\), and the sets \(S_1,\ldots,S_p\) form a disjoint partition of
\(\{1,\ldots,m\}\).

If \(m\leq n\), the orientation map leaves \(M\) unchanged.  Each oriented
row is therefore stored completely on one rank, and applying
\eqref{eq:app-rownorm-eps} to each local row is exactly the corresponding
slice of dense RowNorm.

If \(m>n\), the orientation map gives \(X=M^{\top}\in\R^{n\times m}\).
An oriented row is now split across ranks.  For each original column
\(j\in\{1,\ldots,n\}\), rank \(q\) computes
\begin{equation}
    s_j^{(q)}:=\sum_{i\in S_q}M_{ij}^{2}.
    \label{eq:app-local-sumsq}
\end{equation}
After a sum all-reduce, every rank forms
\begin{equation}
    \nu_j
    :=
    \max\left\{
        \sqrt{\sum_{q=1}^{p}s_j^{(q)}},
        \varepsilon
    \right\}.
    \label{eq:app-global-column-norm}
\end{equation}

Each rank then rescales its local entries as
\begin{equation}
\widehat{M}_{ij}^{(q)}
:=
\frac{M_{ij}}{\nu_j},
\qquad i\in S_q,\quad j=1,\ldots,n.
\end{equation}

\begin{proposition}[Exact-arithmetic equivalence]
\label{prop:app-sharded-rownorm}
Under the partition above, concatenating the local outputs
\(\widehat{M}^{(1)},\ldots,\widehat{M}^{(p)}\)
gives exactly
\[
\mathcal O^{-1}
\bigl(R_\epsilon(\mathcal O(M))\bigr).
\]
\end{proposition}

\begin{proof}
The complete-row case follows immediately because each local denominator is
computed from all entries of its row.  In the split-row case,
\begin{equation}
    \sum_{q=1}^{p}s_j^{(q)}
    =\sum_{q=1}^{p}\sum_{i\in S_q}M_{ij}^{2}
    =\sum_{i=1}^{m}M_{ij}^{2}
    =\|X_{j:}\|_2^2.
\end{equation}
Hence \(\nu_j=\max\{\|X_{j:}\|_2,\varepsilon\}\).  Every local entry of
the oriented row is divided by the same dense denominator, so concatenating
the local pieces yields \(X_{j:}/\nu_j\) for every \(j\), which is exactly
\(\RN_{\varepsilon}(X)\).  Mapping back through \(\mathcal{O}^{-1}\)
completes the proof.
\end{proof}

The proposition is exact over real arithmetic.  In floating-point arithmetic,
the all-reduce may change the summation order and therefore introduce only
the usual reduction-order roundoff differences.

\subsection{Per-rank communication volume}
\label{app:communication-volume}

We report logical algorithmic payload and exclude headers, padding, collective
launch latency, and parameter communication already required by forward and
backward propagation.  Let \(b_M\) be the number of bytes per communicated
momentum element and \(b_S\) the number of bytes per communicated norm
statistic.  Under a bandwidth-optimal ring model, define
\begin{equation}
    \alpha_{\mathrm{AG}}(p):=\frac{p-1}{p},
    \qquad
    \alpha_{\mathrm{AR}}(p):=2\frac{p-1}{p}.
    \label{eq:app-collective-factors}
\end{equation}
The factor of two in \(\alpha_{\mathrm{AR}}(p)\) accounts for the reduce-scatter and all-gather phases of a
ring all-reduce.

For a refresh of \(M\in\R^{m\times n}\), the per-rank all-gather volume is
\begin{equation}
    V_{\mathrm{refresh}}(M)
    =\alpha_{\mathrm{AG}}(p)b_Mmn.
    \label{eq:app-refresh-bytes}
\end{equation}
Every rank executes NS5 on the gathered matrix and retains its local output
slice, so this execution path requires no subsequent scatter of the update.

On an off-refresh step, complete-row RowNorm requires no optimizer-specific
collective.  A tall matrix \(m>n\) requires an all-reduce of \(n\) statistics,
giving
\begin{equation}
    V_{\mathrm{RN}}(M)
    =
    \begin{cases}
        \alpha_{\mathrm{AR}}(p)b_S n, & m>n,\\
        0, & m\leq n.
    \end{cases}
    \label{eq:app-rownorm-bytes}
\end{equation}
For one tall matrix, the ratio is therefore
\begin{equation}
    \frac{V_{\mathrm{RN}}(M)}{V_{\mathrm{refresh}}(M)}
    =\frac{2b_S}{b_Mm}.
    \label{eq:app-single-matrix-byte-ratio}
\end{equation}

For a collection \(\mathcal{I}\) of periodically updated matrices, define
\begin{equation}
    S_{\mathrm{full}}:=\sum_{i\in\mathcal{I}}m_in_i,
    \qquad
    S_{\mathrm{tall}}:=\sum_{i\in\mathcal{I}:m_i>n_i}n_i.
\end{equation}
The period-averaged per-rank payload is
\begin{align}
    \overline V_{\mathrm{rank}}(K)
    &=\frac{1}{K}\alpha_{\mathrm{AG}}(p)b_MS_{\mathrm{full}}
      +\frac{K-1}{K}\alpha_{\mathrm{AR}}(p)b_SS_{\mathrm{tall}}.
    \label{eq:app-average-bytes}
\end{align}
Relative to vanilla Muon's every-step refresh payload,
\begin{equation}
    \frac{\overline V_{\mathrm{rank}}(K)}
    {\alpha_{\mathrm{AG}}(p)b_MS_{\mathrm{full}}}
    =
    \frac{1}{K}
    +\frac{K-1}{K}
      \frac{2b_SS_{\mathrm{tall}}}{b_MS_{\mathrm{full}}}.
    \label{eq:app-average-byte-ratio}
\end{equation}
When the reduced statistics are small relative to the full matrices, the
second term is negligible and the ratio approaches \(1/K\).  This is a
statement about logical payload, not a claim that collective time or total
step time decreases by the same factor.

\subsection{Idealized overlap model}
\label{app:overlap-model}

Bucketed execution changes exposed latency rather than the logical byte count
in~\eqref{eq:app-refresh-bytes}.  Consider \(B\) ordered refresh buckets.  Let
\(a_{\ell}\) be the all-gather time for bucket \(\ell\), and let
\(c_{\ell}\) be its reconstruction, NS5, and slicing time.  In an ideal
two-stage pipeline that overlaps the computation of bucket \(\ell\) with the
all-gather of bucket \(\ell+1\), the makespan is
\begin{equation}
    T_{\mathrm{pipe}}
    =a_1+
      \sum_{\ell=1}^{B-1}\max\{c_{\ell},a_{\ell+1}\}
      +c_B.
    \label{eq:app-pipeline-time}
\end{equation}
The corresponding exposed communication time beyond bucket computation is
\begin{equation}
    T_{\mathrm{comm,exposed}}
    =a_1+
      \sum_{\ell=1}^{B-1}\max\{0,a_{\ell+1}-c_{\ell}\}.
    \label{eq:app-exposed-communication}
\end{equation}
These equations describe an ideal dependency graph.  Actual traces may include
stream synchronization, launch overhead, contention, and imperfect kernel
concurrency; end-to-end claims are therefore based on measured profiler
windows spanning complete periods.

\subsection{Wall-clock break-even condition}
\label{app:wall-clock-break-even}

Let \(N_{\mathrm{base}}(q)\) and \(N_{\mathrm{per}}(q)\) be the numbers of
optimization steps required by vanilla Muon and Periodic Muon, respectively,
to reach a prespecified quality threshold \(q\).  Let their measured mean
step times be \(\tau_{\mathrm{base}}\) and \(\tau_{\mathrm{per}}\), using the
same number of GPUs. Periodic Muon improves time to quality exactly
when
\begin{equation}
    N_{\mathrm{per}}(q)\tau_{\mathrm{per}}
    <N_{\mathrm{base}}(q)\tau_{\mathrm{base}},
\end{equation}
or equivalently,
\begin{equation}
    \frac{N_{\mathrm{per}}(q)}{N_{\mathrm{base}}(q)}
    <
    \frac{\tau_{\mathrm{base}}}{\tau_{\mathrm{per}}}.
    \label{eq:app-time-to-quality-condition}
\end{equation}
Thus a lower per-step optimizer compute and communication cost is insufficient by itself: the realized
speedup must exceed any increase in the number of steps required to reach the
same quality.  If the GPU counts differ, the analogous GPU-hour
comparison multiplies each side by its respective accelerator count.

\section{Algorithms}
\label{app:full-optimizer}

In this section, we provide the Periodic Row-wise Muon algorithm  \ref{alg:periodic-muon} and the system-level algorithm \ref{alg:distributed-periodic-muon}. 

\begin{algorithm}[t]
\caption{Periodic Row-wise Muon}
\label{alg:periodic-muon}
\small
\begin{minipage}{0.98\linewidth}
\textbf{Input:} two-dimensional Muon parameters
\(\mathcal W_{\mathrm M}\), remaining parameters
\(\mathcal W_{\mathrm A}\), period \(K\), and RowNorm multiplier
\(\gamma\).
\par\smallskip
\ResetAlgorithmLines
\AlgLine{\textbf{for} \(t=0,\ldots,T-1\) \textbf{do}}
\AlgLine[1]{Compute stochastic gradients and update every momentum
\(M_t\) as in Section~2.}
\AlgLine[1]{Set
\(\rho_t\leftarrow\mathbb{I}[t\bmod K=0]\).}
\AlgLine[1]{\textbf{for each}
\(W_t\in\mathcal W_{\mathrm M}\) \textbf{do}}
\AlgLine[2]{\(\widetilde M_t\leftarrow\Orient(M_t)\).}
\AlgLine[2]{\textbf{if} \(\rho_t=1\) \textbf{then}
\(\widetilde D_t\leftarrow\NSfive(\widetilde M_t)\).}
\AlgLine[2]{\textbf{else}
\(\widetilde D_t\leftarrow
\gamma\RowNorm_\epsilon(\widetilde M_t)\).}
\AlgLine[2]{\(D_t\leftarrow
s(W_t)\Orient^{-1}(\widetilde D_t)\).}
\AlgLine[2]{\(W_{t+1}\leftarrow
(1-\eta_t\lambda)W_t-\eta_tD_t\).}
\AlgLine[1]{\textbf{end for}}
\AlgLine[1]{Update \(\mathcal W_{\mathrm A}\) with the same AdamW rule
as the baseline.}
\AlgLine{\textbf{end for}}
\end{minipage}
\end{algorithm}

\begin{algorithm}[t]
\caption{Distributed Periodic Muon}
\label{alg:distributed-periodic-muon}
\small
\begin{minipage}{0.98\linewidth}
\textbf{Input:} momentum shards \(\{M_t^{(p)}\}\) on rank \(p\),
refresh flag \(\rho_t\), and communication buckets
\(\{\mathcal B_\ell\}\).
\par\smallskip
\ResetAlgorithmLines
\AlgLine{\textbf{if} \(\rho_t=1\) \textbf{then}
\hfill\textit{// NS5 refresh path}}
\AlgLine[1]{Asynchronously all-gather the first momentum bucket.}
\AlgLine[1]{\textbf{for each} bucket
\(\mathcal B_\ell\) \textbf{do}}
\AlgLine[2]{Wait for the current bucket and launch the next
asynchronous all-gather.}
\AlgLine[2]{Reconstruct full momenta, execute NS5, and slice the results
back to local shards.}
\AlgLine[2]{Release the full-matrix buffers of the current bucket.}
\AlgLine[1]{\textbf{end for}}
\AlgLine{\textbf{else}
\hfill\textit{// sharded RowNorm path}}
\AlgLine[1]{Compute local column-squared statistics for all tall
matrices.}
\AlgLine[1]{Pack the statistics and launch bucketed asynchronous sum
all-reduces.}
\AlgLine[1]{While communication is in flight, compute and finish local
RowNorm updates for the remaining matrices.}
\AlgLine[1]{Wait for all-reduce completion; take the square root and then apply
\(\operatorname{clamp}_{\min}(\epsilon)\).}
\AlgLine[1]{Use the global column norms to finish RowNorm updates for
tall matrices.}
\AlgLine{\textbf{end if}}
\AlgLine{Apply \(\gamma\) to the off-refresh RowNorm outputs only, and fuse
shape scaling with the local parameter-shard updates.}
\end{minipage}
\end{algorithm}

\section{Detailed Experimental Setup}
\label{app:experimental_setup}

\subsection{GPIC Dataset}
\label{app:gpic_dataset}

GPIC (Giant Permissive Image Corpus)~\citep{chandrasegaran2026gpic} is an open image--text dataset designed for large-scale visual-generation research and contains approximately 28 trillion pixels. Its full split comprises 100M training examples, 200K validation examples, and 1M test examples. We train on GPIC-Full, use the validation set to monitor training, and perform generation evaluation on the held-out test set.

GPIC applies vision-language-model-based quality and safety filtering and removes duplicate and near-duplicate images using SSCD visual features. After filtering and deduplication, approximately 101.3M images remain and are divided into 100M training, 200K validation, and 1M test examples. Qwen3-VL-4B-Instruct~\citep{bai2025qwen3} generates captions at multiple levels of detail rather than relying on potentially noisy or missing web metadata. The training, validation, and test splits retain similar source and caption distributions.

\subsection{Model Configurations}
\label{app:model_configurations}

We use a latent space, dual stream MMDiT architecture~\citep{esser2024scaling}. A frozen FLUX.1-schnell~\citep{labs2025flux} VAE encodes images into the latent space, while frozen CLIP-L, CLIP-G~\citep{radford2021learning}, and T5-XXL~\citep{raffel2020exploring} encoders provide text conditioning. The MMDiT backbone includes joint image--text attention, RoPE positional encoding, QK normalization, SwiGLU MLPs, and adaptive normalization. All four models use the same components and connectivity, and differ only in hidden dimension, depth, and number of attention heads.

In the Muon and Periodic Row-wise Muon runs, the corresponding Muon update rule is applied to two-dimensional weight matrices within the Transformer blocks. Biases, normalization parameters, scalar parameters, and parameters outside the Transformer blocks are updated by AdamW.

\subsection{Optimizer and Training Hyperparameters}
\label{app:optimizer_hyperparameters}

All methods use decoupled weight decay. The AdamW baseline applies AdamW to all parameters. Muon and Periodic Row-wise Muon apply matrix-valued updates to two-dimensional hidden-layer weights and maintain a separate AdamW parameter group for all remaining parameters.

All Muon experiments use five Newton--Schulz iterations with fixed coefficients $(a,b,c)=(3.4445,-4.7750,2.0315)$.

At a given model scale, the three optimizers use the same model, training length, global batch size, and learning-rate schedule, and therefore process the same number of training examples. Because their optimizer computation and communication costs differ, they are not constrained to use the same wall-clock time or total compute.

\subsection{Generation Evaluation}
\label{app:generation_evaluation}

We evaluate a checkpoint every 5{,}000 training steps, covering the complete trajectory from 5{,}000 to 60{,}000 steps.
We select a fixed set of 50{,}000 prompts from the GPIC test set and generate 50{,}000 images at $512\times512$ resolution. All models and optimizers use the same sampling configuration with a fixed classifier-free guidance scale of 5.0.

\begin{table*}[t]
    \centering
    \caption{Model configurations. Muon-routed parameters denotes the fraction of trainable parameters updated by Muon in the Muon-based runs.}
    \label{tab:model_configurations}
    \small
    \setlength{\tabcolsep}{6pt}
    \begin{tabular}{@{}lrrrrrr@{}}
        \toprule
        Model & Total params. & Trainable params. & Hidden dim. & Blocks & Heads & Muon-routed \\
        \midrule
        1.3B & 1{,}255{,}586{,}099  & 1{,}255{,}585{,}795  & 1{,}216 & 19 & 19 & 98.69\% \\
        4B   & 3{,}950{,}236{,}099  & 3{,}950{,}235{,}715  & 1{,}920 & 24 & 30 & 99.20\% \\
        9B   & 9{,}374{,}837{,}315  & 9{,}374{,}836{,}803  & 2{,}560 & 32 & 40 & 99.48\% \\
        15B  & 14{,}848{,}237{,}251 & 14{,}848{,}236{,}611 & 2{,}880 & 40 & 45 & 99.60\% \\
        \bottomrule
    \end{tabular}
\end{table*}


\begin{table}[t]
    \centering
    \caption{
        Optimizer hyperparameters.
        ``Aux. AdamW'' denotes the auxiliary AdamW parameter group used in
        the Muon-based runs.
        All methods use weight decay $0.01$ and AdamW-updated parameter groups
        use $\epsilon=10^{-8}$.
    }
    \label{tab:optimizer_hyperparameters}

    \small
    \renewcommand{\arraystretch}{1.12}
    \setlength{\tabcolsep}{4pt}

    \begin{tabular*}{\linewidth}{
        @{\extracolsep{\fill}}
        l
        cccc
        @{}
    }
        \toprule
        & \multicolumn{4}{c}{Main learning rate} \\
        \cmidrule(lr){2-5}
        Method
        & 1.3B
        & 4B
        & 9B
        & 15B \\
        \midrule

        AdamW
        & $3.0{\times}10^{-4}$
        & $2.0{\times}10^{-4}$
        & $1.4{\times}10^{-4}$
        & $1.2{\times}10^{-4}$ \\

        Muon
        & $5.0{\times}10^{-3}$
        & $3.0{\times}10^{-3}$
        & $2.4{\times}10^{-3}$
        & $2.0{\times}10^{-3}$ \\

        Periodic Row-wise Muon
        & $5.0{\times}10^{-3}$
        & $3.0{\times}10^{-3}$
        & $2.4{\times}10^{-3}$
        & $2.0{\times}10^{-3}$ \\

        \bottomrule
    \end{tabular*}

    \vspace{0.8em}

    \begin{tabularx}{\linewidth}{
    @{}
    l
    c
    >{\raggedright\arraybackslash}p{0.22\linewidth}
    >{\raggedright\arraybackslash}X
    @{}
}
    \toprule
    Method
    & \makecell{Main momentum\\or betas}
    & Aux. AdamW
    & Orthogonalization / schedule \\
    \midrule

    AdamW
    & $(0.9,0.95)$
    & --
    & -- \\

    Muon
    & $0.95$
    & LR $=1.0{\times}10^{-4}$;\newline
      $\beta=(0.9,0.95)$
    & NS5 at every step \\

    Periodic Row-wise Muon
    & $0.95$
    & LR $=1.0{\times}10^{-4}$;\newline
      $\beta=(0.9,0.95)$
    & NS5 refresh every $K=3$ steps;\newline
      RowNorm otherwise ($\gamma=0.15$) \\

    \bottomrule
\end{tabularx}
\end{table}

\begin{table}[t]
    \centering
    \caption{Training settings shared by all methods.}
    \label{tab:shared_training_settings}
    \footnotesize
    \renewcommand{\arraystretch}{1.08}
    \setlength{\tabcolsep}{6pt}
    \begin{tabular}{@{}ll@{}}
        \toprule
        Setting & Value \\
        \midrule
        Image resolution          & $512 \times 512$ \\
        Training length           & 60{,}000 steps \\
        Local / global batch size & 16 / 4{,}096 \\
        Learning-rate schedule    & 2k warmup, then cosine decay \\
        Hardware                  & 32 nodes, 256 NVIDIA H100 GPUs \\
        Distributed training      & Pytorch FSDP2 with activation checkpointing \\
        \bottomrule
    \end{tabular}
\end{table}

We evaluate generated images using the following metrics:
\begin{itemize}
    \item \textbf{FD-DINO}v2~\citep{stein2023exposing} computes the Fr\'{e}chet distance between generated and real-image distributions in DINOv2 feature space.
    \item \textbf{FID}~\citep{heusel2017gans} computes the Fr\'{e}chet distance between generated and real-image distributions in Inception feature space.
    \item \textbf{Maximum Mean Discrepancy (MMD)} provides a non-parametric measure of the discrepancy between generated and real feature distributions.
    \item \textbf{Precision and Density} measure how closely generated samples align with the real-data manifold and primarily characterize sample fidelity.
    \item \textbf{Recall and Coverage} measure how well the generated distribution covers the real data distribution and characterize generation diversity.
    \item \textbf{HPSv2.1}~\citep{wu2023human} measures text--image alignment and consistency with human preferences.
    \item \textbf{GenEval2}~\citep{kamath2025geneval} evaluates the correctness of generated objects, attributes, relations, and compositional semantics.
\end{itemize}

Unless otherwise specified, all main comparisons and ablations use the same generation and evaluation protocol.

\section{Additional Experimental Results}
\label{app:additional_experiments}

This section provides additional experimental results supporting the main
comparisons. We first describe the selection of the refresh period \(K\) and
the RowNorm scaling factor \(\gamma\) for Periodic Row-wise Muon.
We then provide
the complete training loss trajectories, checkpoint-level generation
evaluations, and additional profiler traces. 


\subsection{Selection of \(K\) and \(\gamma\)}
\label{app:k_gamma_selection}

Periodic Row-wise Muon introduces two primary hyperparameters: the refresh
period \(K\), which controls the frequency of full NS5 spectral updates, and
the RowNorm scaling factor \(\gamma\), which controls the relative update
magnitude of the RowNorm branch on non-refresh steps. To avoid tuning these
hyperparameters separately at each model scale, we select \(K\) and
\(\gamma\) once in a smaller development setting and directly transfer the
same configuration to all 1.3B--15B models in the main experiments without
additional scale-specific tuning.

Specifically, we perform the hyperparameter selection experiments using the
1.3B model for 30k training steps on 2 H100 nodes, corresponding to 16 H100
GPUs in total, with a local batch size of 16. We otherwise retain the same
model, optimizer hyperparameters, learning rate schedule, data pipeline, and
training recipe used in the main experiments. We evaluate
\(K\in\{2,3,4\}\) and
\(\gamma\in\{0.10,0.15,0.25,0.35\}\). To reduce sensitivity to noise at any
single checkpoint, we use the mean FD-DINO over the 20k, 25k, and 30k
checkpoints as the selection criterion.

Figure~\ref{fig:k_gamma_selection} summarizes the results. For \(K=3\) cases,
\(\gamma=0.15\) achieves the lowest late-stage mean FD-DINO among the tested
scales. We therefore use \(\gamma=0.15\) when comparing different refresh
periods. Increasing \(K\) from 2 to 3 reduces the number of NS5 refreshes by
33\% while increasing the best mean FD-DINO by only 2.8\%. In contrast,
increasing \(K\) from 3 to 4 provides a further 25\% reduction in NS5
refreshes but increases mean FD-DINO by 8.2\%. We therefore select
\(K=3\) and \(\gamma=0.15\) as the default configuration, which provides a
favorable trade-off between generation quality and spectral-update compute and communication cost.

Importantly, this selection procedure does not retune \(K\) or \(\gamma\)
on the larger models. The same configuration is transferred directly to all
model scales in the main experiments, so the observed scaling behavior does
not result from scale-specific hyperparameter tuning.

\begin{figure*}[t]
    \centering
    \includegraphics[
        width=\textwidth
    ]{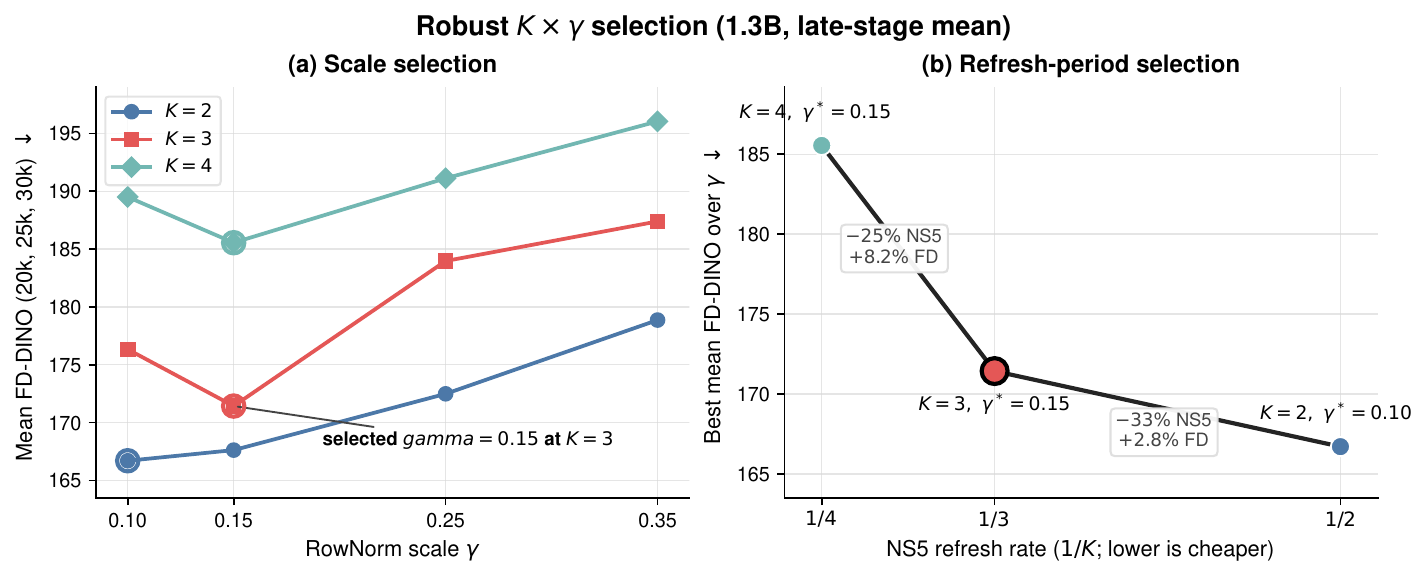}
    \caption{
        Selection of \(K\) and \(\gamma\).
        (a) Late-stage FD-DINO, averaged over the 20k, 25k, and
        30k checkpoints, across different \(K\) and \(\gamma\).
        (b) The best mean FD-DINO over \(\gamma\) at each refresh
        period. Increasing \(K\) from 2 to 3 reduces the number of
        NS5 refreshes by \(33\%\) with only a \(2.8\%\) increase in
        FD-DINO, whereas increasing \(K\) from 3 to 4 saves a further
        \(25\%\) but increases FD-DINO by \(8.2\%\).
        We therefore select \(K=3\) and \(\gamma=0.15\).
    }
    \label{fig:k_gamma_selection}
\end{figure*}

\subsection{Training-Loss Trajectories Across Model Scales}
\label{app:training_loss}

Figure~\ref{fig:all_training_curve} shows the complete training-loss
trajectories of AdamW, Muon, and Periodic Row-wise Muon over 60k
optimization steps for the 1.3B, 4B, 9B, and 15B models. These results
complement the validation loss comparisons in the main text by providing a
higher resolution view of the optimization dynamics throughout training.

All three optimizers exhibit stable decreasing training loss across all
model scales, with no evidence of training divergence. Muon generally
achieves lower training loss throughout the main training regime, while
Periodic Row-wise Muon closely tracks the Muon trajectory. This indicates
that replacing two out of every three NS5 updates with the substantially
cheaper RowNorm update does not destabilize optimization. As model size
increases from 1.3B to 15B, the overall training loss level also decreases,
consistent with the validation loss scaling behavior observed in the main
experiments.

Because minibatch training loss exhibits substantial stochastic variation,
we do not use individual training loss values as proxies for generation
quality. Generation performance is instead evaluated using FD-DINO, FID,
and the additional generation metrics reported in the main text and in
Appendix~\ref{app:checkpoint_evals}. The purpose of
Figure~\ref{fig:all_training_curve} is therefore to verify that
Periodic Row-wise Muon remains stable over the complete 60k-step
optimization trajectory rather than only at the sparse generation evaluation
checkpoints.

\begin{figure*}[t]
    \centering
    \includegraphics[
        width=\textwidth,
        trim={0 76bp 176bp 0},
        clip
    ]{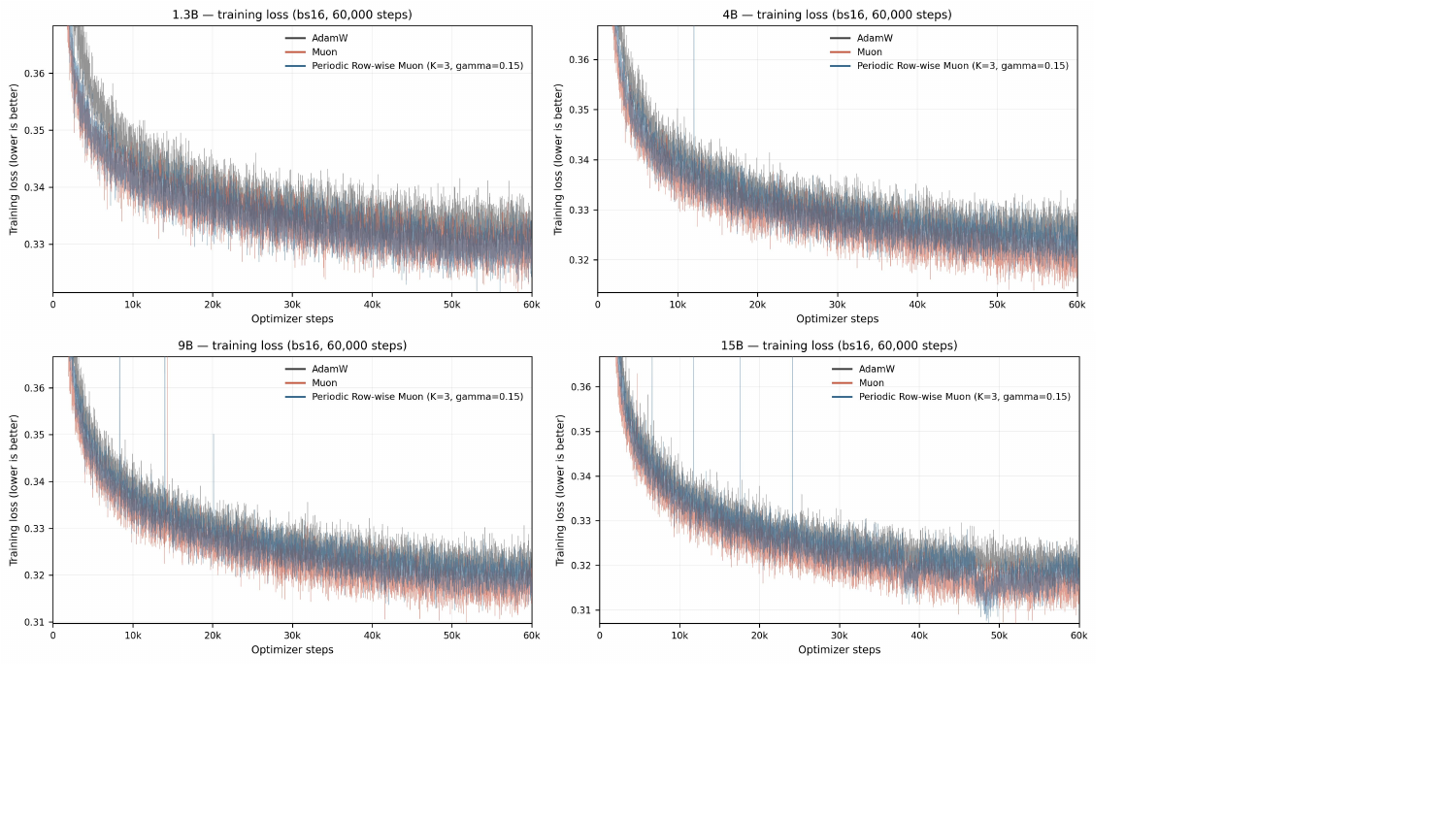}
    \caption{Training loss curves across model scales over 60,000 optimizer steps.}
    \label{fig:all_training_curve}
\end{figure*}

\subsection{Active Time Efficiency over AdamW}
\label{app:adamw_periodic_active_time}

\begin{figure*}[t]
    \centering
    \includegraphics[width=\textwidth]{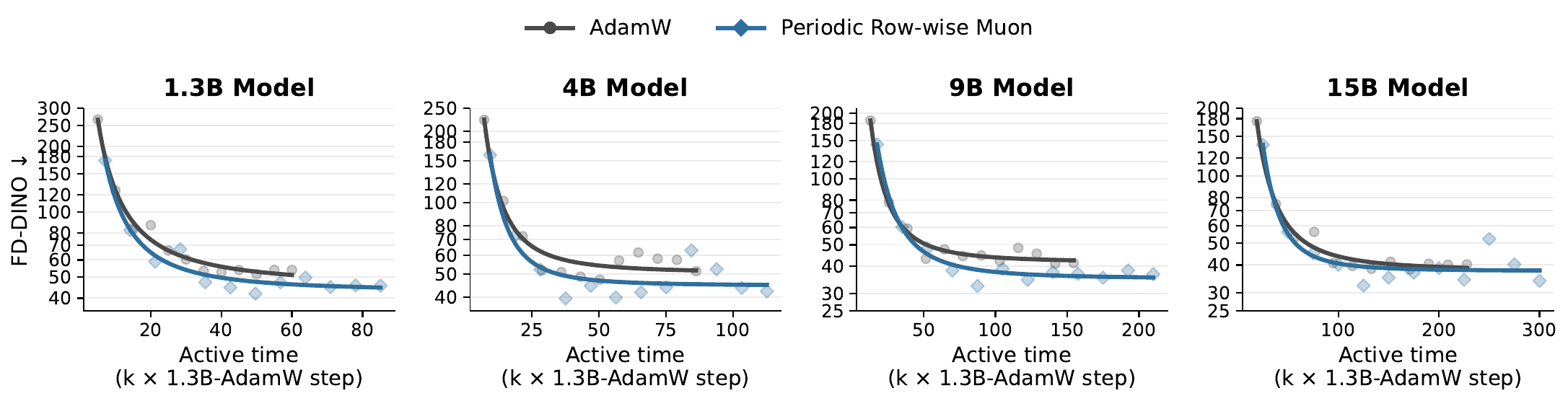}
    \caption{
    Generation quality for AdamW and Periodic Row-wise Muon.
    }
    \label{fig:adamw_periodic_active_time}
\end{figure*}

Figure~\ref{fig:adamw_periodic_active_time} directly compares Periodic
Row-wise Muon with AdamW under active training time. Across all model sizes,
Periodic Row-wise Muon achieves lower FD-DINO throughout most of the
overlapping time range, indicating better generation quality for a comparable
training time budget. It also consistently attains lower FD-DINO at the end of
training. Thus, the advantage over AdamW holds in both training time efficiency
and the final generation quality reached by the model.

\subsection{Complete Checkpoint Evaluation Results}
\label{app:checkpoint_evals}

The main text reports the most important final checkpoint metrics and the
best observed FD-DINO to keep the presentation compact. Here we provide the
complete checkpoint-level generation evaluations for all 4 model scales
in Tables~\ref{tab:all_checkpoint_metrics_1b}--\ref{tab:all_checkpoint_metrics_10b}.

For each model, we evaluate checkpoints every 5k steps from 5k to 60k using
the same evaluation protocol for AdamW, Muon, and Periodic Row-wise Muon.
At every checkpoint, we report FD-DINO, FID, MMD-DINO, Precision, Recall,
Coverage, Density, HPSv2, and both the arithmetic and geometric means of
GenEval2. These tables therefore complement the fidelity metrics emphasized
in the main text with the complete evolution of diversity, alignment, and
compositionality throughout training.

The full trajectories show that both Muon and Periodic Row-wise Muon achieve
substantial improvements over AdamW at multiple stages of training, while
the relative ordering of the two Muon variants exhibits some checkpoint-level
variation. In particular, the checkpoint that optimizes one generation
metric does not necessarily optimize the others. We therefore report the
complete metric profile at the final checkpoint in the main text and treat
the best-observed FD-DINO separately, rather than constructing an artificial
``best'' model by selecting different checkpoints for different metrics.

These trajectories also provide additional context for the generation quality
frontiers in the main text. Different optimizers enter their best quality
regimes at different points in training, and Periodic Row-wise Muon typically
reaches a generation quality regime comparable to vanilla Muon with
substantially less normalized active training time.

\subsection{Additional Profiler Traces Across Model Scales}
\label{app:additional_profiler}

The main text uses the 15B model to illustrate the systems behavior of
Periodic Row-wise Muon. Figure~\ref{fig:profiler_additional}
provides the corresponding profiler traces for the 1.3B, 4B, and 9B
models, allowing us to verify whether the same execution pattern persists
across model scales.

All profiler traces use the same global time normalization as the systems
measurements in the main text, meaning that the mean step time of the 1.3B AdamW run is
defined as one time unit. The traces separately visualize the non-optimizer
phase, optimizer-specific CPU phases, GPU computation, and communication activity for vanilla Muon and Periodic Row-wise Muon.

A consistent execution pattern is observed at all three scales. Vanilla
Muon performs a full NS5 update at every optimization step, resulting in a
long spectral computation phase together with the corresponding optimizer
communication. Periodic Row-wise Muon instead retains only one refresh step
per \(K=3\) period, while the two intervening steps use the much shorter
RowNorm path. On refresh steps, bucketed all-gather overlaps communication
for upcoming matrices with NS5 computation on matrices that are already
available. On RowNorm steps, the much smaller communication of norm
statistics is similarly overlapped with local computation.

Over the three consecutive steps shown in
Figure~\ref{fig:profiler_additional}, the normalized time
decreases from 5.59 for vanilla Muon to 4.20 for Periodic Row-wise Muon on
the 1.3B model, from 7.26 to 5.68 on the 4B model, and from 12.47 to 10.26
on the 9B model. The 15B trace presented in the main text exhibits the same
qualitative behavior. The consistency of these profiler traces across model
scales indicates that the systems gains do not arise from a single model
shape, but from the combined effects of periodic spectral refresh, sharded
RowNorm, and communication--computation overlap.

\begin{table*}[p]
  \centering
  \caption{All checkpoint evaluation metrics for the 1.3B model.}
  \label{tab:all_checkpoint_metrics_1b}
  \scriptsize
  \setlength{\tabcolsep}{2.4pt}
  \renewcommand{\arraystretch}{1.02}
  \resizebox{\textwidth}{!}{%
  \begin{tabular}{@{}ll
    D{.}{.}{3.2} D{.}{.}{2.2} D{.}{.}{1.3}
    D{.}{.}{1.3} D{.}{.}{1.3} D{.}{.}{1.3} D{.}{.}{1.3}
    D{.}{.}{2.2}
    D{.}{.}{2.2} D{.}{.}{2.2}@{}}
    \toprule
    \multirow{2}{*}{Optimizer} & \multirow{2}{*}{Step} &
    \multicolumn{3}{c}{Fidelity} & \multicolumn{4}{c}{Diversity} &
    \multicolumn{1}{c}{Alignment} & \multicolumn{2}{c}{Compositionality} \\
    \cmidrule(lr){3-5}\cmidrule(lr){6-9}\cmidrule(lr){10-10}\cmidrule(lr){11-12}
    & & \multicolumn{1}{c}{FD-DINO$\downarrow$} & \multicolumn{1}{c}{FID$\downarrow$} & \multicolumn{1}{c}{MMD-DINO$\downarrow$} &
    \multicolumn{1}{c}{Precision$\uparrow$} & \multicolumn{1}{c}{Recall$\uparrow$} & \multicolumn{1}{c}{Coverage$\uparrow$} & \multicolumn{1}{c}{Density$\uparrow$} &
    \multicolumn{1}{c}{HPSv2$\uparrow$} & \multicolumn{1}{c}{GenEval2 AM$\uparrow$} & \multicolumn{1}{c}{GenEval2 GM$\uparrow$} \\
    \midrule
    \multirow{12}{*}{AdamW} & 5k & 266.40 & 9.94 & 0.258 & 0.861 & 0.361 & 0.657 & 0.900 & 15.79 & 17.19 & 1.35 \\
     & 10k & 125.51 & 7.57 & 0.094 & 0.908 & 0.681 & 0.842 & 0.947 & 15.93 & 22.24 & 2.77 \\
     & 15k & 82.93 & 6.62 & 0.051 & 0.915 & 0.789 & 0.880 & 0.976 & 18.09 & 28.53 & 4.00 \\
     & 20k & 86.84 & 7.94 & 0.055 & 0.925 & 0.812 & 0.890 & 1.107 & 17.71 & 29.83 & 4.19 \\
     & 25k & 66.30 & 6.58 & 0.037 & 0.929 & 0.849 & 0.921 & 0.999 & 17.55 & 30.96 & 4.75 \\
     & 30k & 60.34 & 6.89 & 0.034 & 0.941 & 0.859 & 0.929 & 1.065 & 19.56 & 36.92 & 6.71 \\
     & 35k & 53.43 & 6.33 & 0.029 & 0.938 & 0.874 & 0.934 & 1.019 & 19.92 & 39.37 & 8.61 \\
     & 40k & 52.60 & 6.30 & 0.030 & 0.947 & 0.879 & 0.947 & 1.018 & 19.82 & 37.60 & 7.16 \\
     & 45k & 53.92 & 6.53 & 0.030 & 0.943 & 0.881 & 0.940 & 0.998 & 19.72 & 37.97 & 7.57 \\
     & 50k & 51.56 & 6.68 & 0.028 & 0.947 & 0.887 & 0.940 & 1.019 & 20.61 & 38.16 & 8.01 \\
     & 55k & 53.97 & 6.61 & 0.031 & 0.941 & 0.886 & 0.941 & 1.018 & 20.46 & 37.64 & 6.01 \\
     & 60k & 53.93 & 6.71 & 0.031 & 0.941 & 0.887 & 0.940 & 0.966 & 20.53 & 38.66 & 7.10 \\
    \midrule
    \multirow{12}{*}{Muon} & 5k & 152.03 & 7.68 & 0.126 & 0.882 & 0.633 & 0.778 & 0.882 & 16.64 & 20.28 & 2.23 \\
     & 10k & 73.35 & 6.26 & 0.040 & 0.921 & 0.814 & 0.889 & 0.939 & 18.79 & 30.57 & 4.82 \\
     & 15k & 51.56 & 5.73 & 0.024 & 0.943 & 0.866 & 0.920 & 0.982 & 19.69 & 38.33 & 7.45 \\
     & 20k & 48.06 & 6.16 & 0.021 & 0.948 & 0.870 & 0.930 & 0.993 & 19.79 & 43.63 & 9.64 \\
     & 25k & 46.12 & 6.50 & 0.022 & 0.947 & 0.884 & 0.932 & 0.980 & 20.43 & 40.00 & 7.65 \\
     & 30k & 52.86 & 6.84 & 0.029 & 0.942 & 0.885 & 0.938 & 0.975 & 20.66 & 45.20 & 8.05 \\
     & 35k & 46.65 & 6.26 & 0.022 & 0.944 & 0.896 & 0.943 & 0.991 & 21.14 & 46.57 & 10.77 \\
     & 40k & 41.73 & 6.01 & 0.021 & 0.949 & 0.901 & 0.947 & 0.984 & 20.91 & 46.73 & 11.36 \\
     & 45k & 48.26 & 6.56 & 0.027 & 0.943 & 0.897 & 0.945 & 0.991 & 20.29 & 36.27 & 5.45 \\
     & 50k & 50.41 & 6.74 & 0.030 & 0.945 & 0.892 & 0.954 & 0.994 & 20.41 & 45.32 & 8.98 \\
     & 55k & 47.92 & 6.65 & 0.028 & 0.945 & 0.900 & 0.950 & 0.985 & 20.81 & 44.91 & 9.22 \\
     & 60k & 46.08 & 6.54 & 0.027 & 0.949 & 0.907 & 0.949 & 0.992 & 20.84 & 44.72 & 9.38 \\
    \midrule
    \multirow{12}{*}{\shortstack[l]{Periodic Row-wise\\Muon}} & 5k & 172.46 & 8.23 & 0.145 & 0.885 & 0.569 & 0.766 & 0.921 & 16.45 & 19.46 & 1.67 \\
     & 10k & 82.29 & 6.74 & 0.049 & 0.923 & 0.808 & 0.880 & 0.918 & 17.57 & 28.50 & 4.80 \\
     & 15k & 58.96 & 6.06 & 0.028 & 0.928 & 0.858 & 0.916 & 0.969 & 18.45 & 36.49 & 6.03 \\
     & 20k & 67.22 & 6.97 & 0.039 & 0.939 & 0.856 & 0.916 & 1.000 & 18.46 & 34.74 & 5.83 \\
     & 25k & 47.37 & 6.18 & 0.021 & 0.938 & 0.876 & 0.938 & 0.997 & 19.36 & 36.95 & 7.24 \\
     & 30k & 44.77 & 6.11 & 0.021 & 0.940 & 0.891 & 0.932 & 0.936 & 21.10 & 42.73 & 8.92 \\
     & 35k & 41.91 & 5.53 & 0.019 & 0.944 & 0.902 & 0.944 & 0.981 & 20.73 & 47.71 & 10.53 \\
     & 40k & 47.29 & 6.63 & 0.028 & 0.946 & 0.887 & 0.938 & 0.971 & 21.40 & 45.68 & 11.36 \\
     & 45k & 49.77 & 6.37 & 0.030 & 0.947 & 0.894 & 0.944 & 0.981 & 20.18 & 45.97 & 10.21 \\
     & 50k & 45.01 & 6.13 & 0.026 & 0.951 & 0.893 & 0.949 & 0.996 & 20.31 & 47.40 & 11.25 \\
     & 55k & 45.75 & 6.40 & 0.027 & 0.949 & 0.896 & 0.951 & 1.005 & 20.69 & 47.93 & 11.71 \\
     & 60k & 45.65 & 6.29 & 0.027 & 0.952 & 0.906 & 0.949 & 0.976 & 20.36 & 45.89 & 9.75 \\
    \bottomrule
  \end{tabular}%
  }
  \vspace{2pt}
\end{table*}

\clearpage

\begin{table*}[p]
  \centering
  \caption{All checkpoint evaluation metrics for the 4B model.}
  \label{tab:all_checkpoint_metrics_3b}
  \scriptsize
  \setlength{\tabcolsep}{2.4pt}
  \renewcommand{\arraystretch}{1.02}
  \resizebox{\textwidth}{!}{%
  \begin{tabular}{@{}ll
    D{.}{.}{3.2} D{.}{.}{2.2} D{.}{.}{1.3}
    D{.}{.}{1.3} D{.}{.}{1.3} D{.}{.}{1.3} D{.}{.}{1.3}
    D{.}{.}{2.2}
    D{.}{.}{2.2} D{.}{.}{2.2}@{}}
    \toprule
    \multirow{2}{*}{Optimizer} & \multirow{2}{*}{Step} &
    \multicolumn{3}{c}{Fidelity} & \multicolumn{4}{c}{Diversity} &
    \multicolumn{1}{c}{Alignment} & \multicolumn{2}{c}{Compositionality} \\
    \cmidrule(lr){3-5}\cmidrule(lr){6-9}\cmidrule(lr){10-10}\cmidrule(lr){11-12}
    & & \multicolumn{1}{c}{FD-DINO$\downarrow$} & \multicolumn{1}{c}{FID$\downarrow$} & \multicolumn{1}{c}{MMD-DINO$\downarrow$} &
    \multicolumn{1}{c}{Precision$\uparrow$} & \multicolumn{1}{c}{Recall$\uparrow$} & \multicolumn{1}{c}{Coverage$\uparrow$} & \multicolumn{1}{c}{Density$\uparrow$} &
    \multicolumn{1}{c}{HPSv2$\uparrow$} & \multicolumn{1}{c}{GenEval2 AM$\uparrow$} & \multicolumn{1}{c}{GenEval2 GM$\uparrow$} \\
    \midrule
    \multirow{12}{*}{AdamW} & 5k & 223.73 & 9.17 & 0.199 & 0.870 & 0.453 & 0.700 & 0.901 & 15.90 & 18.62 & 1.78 \\
     & 10k & 101.76 & 8.42 & 0.066 & 0.908 & 0.746 & 0.857 & 0.950 & 18.17 & 22.01 & 2.26 \\
     & 15k & 72.16 & 7.67 & 0.036 & 0.922 & 0.834 & 0.885 & 0.958 & 18.31 & 34.18 & 6.51 \\
     & 20k & 51.87 & 6.87 & 0.023 & 0.941 & 0.866 & 0.926 & 0.998 & 20.73 & 37.09 & 5.76 \\
     & 25k & 50.94 & 6.39 & 0.023 & 0.938 & 0.879 & 0.930 & 1.007 & 19.36 & 32.85 & 5.71 \\
     & 30k & 48.71 & 6.48 & 0.022 & 0.940 & 0.888 & 0.936 & 1.026 & 20.24 & 38.09 & 7.74 \\
     & 35k & 47.28 & 6.80 & 0.024 & 0.946 & 0.888 & 0.944 & 1.014 & 21.60 & 40.01 & 7.50 \\
     & 40k & 56.97 & 7.90 & 0.029 & 0.933 & 0.899 & 0.937 & 0.972 & 20.63 & 41.74 & 8.89 \\
     & 45k & 61.61 & 8.72 & 0.033 & 0.944 & 0.897 & 0.935 & 1.020 & 21.31 & 45.90 & 10.68 \\
     & 50k & 58.04 & 8.28 & 0.030 & 0.939 & 0.899 & 0.941 & 0.994 & 20.72 & 44.26 & 9.72 \\
     & 55k & 57.35 & 8.01 & 0.030 & 0.937 & 0.902 & 0.946 & 0.982 & 20.80 & 44.03 & 9.70 \\
     & 60k & 51.39 & 7.35 & 0.027 & 0.937 & 0.899 & 0.939 & 0.990 & 21.40 & 45.70 & 9.13 \\
    \midrule
    \multirow{12}{*}{Muon}
    & 5k  & 149.32 & 7.82 & 0.119 & 0.888 & 0.631 & 0.804 & 0.894 & 16.73 & 19.54 & 2.31 \\
    & 10k & 65.88  & 6.64 & 0.036 & 0.929 & 0.833 & 0.904 & 0.951 & 18.27 & 27.73 & 4.12 \\
    & 15k & 56.08  & 6.50 & 0.030 & 0.939 & 0.869 & 0.929 & 0.974 & 20.39 & 40.99 & 7.98 \\
    & 20k & 40.15  & 6.25 & 0.015 & 0.941 & 0.903 & 0.935 & 0.952 & 20.81 & 41.14 & 8.46 \\
    & 25k & 39.27  & 5.32 & 0.018 & 0.945 & 0.902 & 0.943 & 0.967 & 21.00 & 46.20 & 10.41 \\
    & 30k & 40.50  & 5.94 & 0.018 & 0.941 & 0.904 & 0.938 & 0.959 & 20.81 & 46.53 & 9.52 \\
    & 35k & 41.96  & 5.66 & 0.020 & 0.946 & 0.905 & 0.947 & 0.983 & 19.91 & 43.45 & 7.41 \\
    & 40k & 55.82  & 6.99 & 0.033 & 0.935 & 0.910 & 0.936 & 0.948 & 18.94 & 39.77 & 6.36 \\
    & 45k & 44.39  & 5.86 & 0.025 & 0.940 & 0.913 & 0.946 & 0.963 & 20.79 & 46.75 & 8.18 \\
    & 50k & 46.24  & 6.13 & 0.027 & 0.946 & 0.914 & 0.946 & 0.967 & 21.13 & 46.42 & 8.83 \\
    & 55k & 42.09  & 6.14 & 0.022 & 0.946 & 0.912 & 0.941 & 0.966 & 21.13 & 50.40 & 10.19 \\
    & 60k & 41.62  & 6.08 & 0.024 & 0.939 & 0.914 & 0.943 & 0.964 & 20.66 & 46.58 & 9.06 \\
    \midrule
    \multirow{12}{*}{\shortstack[l]{Periodic Row-wise\\Muon}} & 5k & 159.06 & 9.04 & 0.125 & 0.892 & 0.594 & 0.799 & 0.969 & 15.61 & 18.68 & 2.05 \\
     & 10k & 72.59 & 7.20 & 0.040 & 0.929 & 0.820 & 0.888 & 0.976 & 18.98 & 28.42 & 3.85 \\
     & 15k & 52.19 & 6.35 & 0.022 & 0.929 & 0.876 & 0.917 & 0.946 & 18.25 & 35.52 & 6.20 \\
     & 20k & 39.42 & 5.61 & 0.015 & 0.949 & 0.899 & 0.948 & 0.984 & 21.05 & 42.98 & 8.16 \\
     & 25k & 44.58 & 6.27 & 0.021 & 0.935 & 0.901 & 0.931 & 0.937 & 20.87 & 38.63 & 7.19 \\
     & 30k & 39.80 & 6.15 & 0.018 & 0.951 & 0.909 & 0.943 & 0.957 & 20.23 & 39.53 & 6.64 \\
     & 35k & 41.91 & 6.32 & 0.020 & 0.941 & 0.910 & 0.941 & 0.965 & 20.83 & 40.87 & 6.95 \\
     & 40k & 43.93 & 6.64 & 0.020 & 0.938 & 0.909 & 0.942 & 0.978 & 20.47 & 43.90 & 8.09 \\
     & 45k & 62.89 & 8.68 & 0.032 & 0.933 & 0.898 & 0.942 & 1.092 & 19.52 & 43.23 & 8.13 \\
     & 50k & 52.36 & 7.30  & 0.030 & 0.936 & 0.916 & 0.942 & 0.939 & 20.54 & 49.32 & 10.86 \\
     & 55k & 43.70 & 6.42 & 0.023 & 0.941 & 0.911 & 0.950 & 0.986 & 21.00 & 49.41 & 11.44 \\
     & 60k & 42.25 & 6.35  & 0.022 & 0.941 & 0.906 & 0.952 & 0.984 & 21.34 & 49.38 & 10.58 \\
    \bottomrule
  \end{tabular}%
  }
  \vspace{2pt}
\end{table*}

\clearpage

\begin{table*}[p]
  \centering
  \caption{All checkpoint evaluation metrics for the 9B model.}
  \label{tab:all_checkpoint_metrics_7b}
  \scriptsize
  \setlength{\tabcolsep}{2.4pt}
  \renewcommand{\arraystretch}{1.02}
  \resizebox{\textwidth}{!}{%
  \begin{tabular}{@{}ll
    D{.}{.}{3.2} D{.}{.}{2.2} D{.}{.}{1.3}
    D{.}{.}{1.3} D{.}{.}{1.3} D{.}{.}{1.3} D{.}{.}{1.3}
    D{.}{.}{2.2}
    D{.}{.}{2.2} D{.}{.}{2.2}@{}}
    \toprule
    \multirow{2}{*}{Optimizer} & \multirow{2}{*}{Step} &
    \multicolumn{3}{c}{Fidelity} & \multicolumn{4}{c}{Diversity} &
    \multicolumn{1}{c}{Alignment} & \multicolumn{2}{c}{Compositionality} \\
    \cmidrule(lr){3-5}\cmidrule(lr){6-9}\cmidrule(lr){10-10}\cmidrule(lr){11-12}
    & & \multicolumn{1}{c}{FD-DINO$\downarrow$} & \multicolumn{1}{c}{FID$\downarrow$} & \multicolumn{1}{c}{MMD-DINO$\downarrow$} &
    \multicolumn{1}{c}{Precision$\uparrow$} & \multicolumn{1}{c}{Recall$\uparrow$} & \multicolumn{1}{c}{Coverage$\uparrow$} & \multicolumn{1}{c}{Density$\uparrow$} &
    \multicolumn{1}{c}{HPSv2$\uparrow$} & \multicolumn{1}{c}{GenEval2 AM$\uparrow$} & \multicolumn{1}{c}{GenEval2 GM$\uparrow$} \\
    \midrule
    \multirow{12}{*}{AdamW} & 5k & 184.83 & 8.37 & 0.150 & 0.883 & 0.531 & 0.757 & 0.943 & 16.08 & 19.18 & 1.75 \\
     & 10k & 78.05 & 6.05 & 0.043 & 0.924 & 0.794 & 0.890 & 0.991 & 19.32 & 27.01 & 3.41 \\
     & 15k & 59.30 & 7.46 & 0.028 & 0.940 & 0.857 & 0.916 & 0.966 & 20.65 & 35.76 & 6.14 \\
     & 20k & 43.20 & 6.27 & 0.016 & 0.944 & 0.881 & 0.933 & 0.967 & 20.58 & 41.28 & 8.23 \\
     & 25k & 47.67 & 6.43 & 0.019 & 0.944 & 0.879 & 0.940 & 1.020 & 20.21 & 40.96 & 8.46 \\
     & 30k & 44.42 & 6.46 & 0.021 & 0.945 & 0.895 & 0.946 & 1.012 & 20.38 & 40.93 & 8.09 \\
     & 35k & 44.72 & 6.58 & 0.022 & 0.942 & 0.902 & 0.942 & 0.979 & 21.70 & 47.67 & 10.16 \\
     & 40k & 42.21 & 6.44 & 0.021 & 0.944 & 0.905 & 0.949 & 0.975 & 20.82 & 44.33 & 9.95 \\
     & 45k & 48.43 & 7.12 & 0.025 & 0.935 & 0.902 & 0.941 & 0.970 & 21.18 & 44.51 & 9.40 \\
     & 50k & 45.54 & 6.77 & 0.025 & 0.942 & 0.905 & 0.942 & 0.958 & 20.54 & 42.45 & 8.89 \\
     & 55k & 41.01 & 6.18 & 0.021 & 0.946 & 0.911 & 0.942 & 0.957 & 21.18 & 45.62 & 10.33 \\
     & 60k & 41.26 & 6.19 & 0.022 & 0.946 & 0.904 & 0.950 & 0.987 & 21.39 & 45.93 & 9.66 \\
    \midrule
    \multirow{12}{*}{Muon} & 5k & 117.21 & 6.48 & 0.085 & 0.905 & 0.706 & 0.836 & 0.937 & 16.82 & 22.76 & 2.21 \\
     & 10k & 47.64 & 5.51 & 0.019 & 0.942 & 0.870 & 0.924 & 0.981 & 20.39 & 37.70 & 8.33 \\
     & 15k & 40.00 & 5.77 & 0.017 & 0.937 & 0.897 & 0.938 & 0.958 & 21.97 & 46.52 & 11.82 \\
     & 20k & 38.57 & 6.46 & 0.018 & 0.947 & 0.903 & 0.942 & 0.952 & 21.93 & 53.00 & 14.17 \\
     & 25k & 38.97 & 5.95 & 0.018 & 0.943 & 0.911 & 0.943 & 0.974 & 20.49 & 49.81 & 12.02 \\
     & 30k & 39.07 & 6.13 & 0.021 & 0.945 & 0.911 & 0.945 & 0.952 & 21.38 & 50.16 & 10.04 \\
     & 35k & 39.30 & 5.99 & 0.020 & 0.950 & 0.910 & 0.941 & 0.966 & 22.27 & 54.65 & 14.14 \\
     & 40k & 35.92 & 5.79 & 0.019 & 0.944 & 0.915 & 0.944 & 0.951 & 22.03 & 47.49 & 10.30 \\
     & 45k & 36.29 & 5.78 & 0.020 & 0.950 & 0.918 & 0.949 & 0.967 & 21.98 & 52.21 & 11.92 \\
     & 50k & 34.53 & 5.71 & 0.020 & 0.948 & 0.922 & 0.947 & 0.947 & 22.87 & 53.98 & 12.86 \\
     & 55k & 34.41 & 5.64 & 0.020 & 0.950 & 0.920 & 0.949 & 0.964 & 22.36 & 51.55 & 11.34 \\
     & 60k & 33.97 & 5.61 & 0.020 & 0.952 & 0.927 & 0.953 & 0.960 & 22.60 & 51.77 & 11.35 \\
    \midrule
    \multirow{12}{*}{\shortstack[l]{Periodic Row-wise\\Muon}} & 5k & 143.77 & 8.71 & 0.108 & 0.903 & 0.630 & 0.824 & 0.960 & 15.83 & 19.38 & 2.29 \\
     & 10k & 60.29 & 6.70 & 0.028 & 0.923 & 0.850 & 0.914 & 0.946 & 18.54 & 29.53 & 4.36 \\
     & 15k & 48.17 & 6.25 & 0.021 & 0.937 & 0.884 & 0.931 & 0.945 & 20.46 & 40.09 & 8.10 \\
     & 20k & 38.14 & 5.72 & 0.014 & 0.947 & 0.896 & 0.935 & 0.946 & 22.12 & 45.39 & 9.82 \\
     & 25k & 32.44 & 5.63 & 0.011 & 0.943 & 0.907 & 0.947 & 0.954 & 21.83 & 51.25 & 11.75 \\
     & 30k & 38.81 & 6.17 & 0.017 & 0.950 & 0.904 & 0.945 & 0.970 & 20.95 & 49.46 & 10.28 \\
     & 35k & 34.57 & 5.63 & 0.016 & 0.952 & 0.918 & 0.948 & 0.979 & 22.08 & 50.73 & 10.89 \\
     & 40k & 37.47 & 5.89 & 0.019 & 0.950 & 0.911 & 0.951 & 0.964 & 22.34 & 54.97 & 13.50 \\
     & 45k & 36.86 & 6.09 & 0.020 & 0.946 & 0.916 & 0.946 & 0.965 & 22.82 & 57.35 & 15.35 \\
     & 50k & 35.37 & 5.85 & 0.019 & 0.952 & 0.912 & 0.952 & 0.972 & 22.57 & 56.16 & 13.44 \\
     & 55k & 38.19 & 6.39 & 0.020 & 0.943 & 0.921 & 0.954 & 0.967 & 22.47 & 55.52 & 12.18 \\
     & 60k & 36.70 & 6.13 & 0.019 & 0.949 & 0.913 & 0.951 & 0.979 & 22.44 & 57.33 & 15.97 \\
    \bottomrule
  \end{tabular}%
  }
  \vspace{2pt}
\end{table*}

\clearpage

\begin{table*}[]
  \centering
  \caption{All checkpoint evaluation metrics for the 15B model.}
  \label{tab:all_checkpoint_metrics_10b}
  \scriptsize
  \setlength{\tabcolsep}{2.4pt}
  \renewcommand{\arraystretch}{1.02}
  \resizebox{\textwidth}{!}{%
  \begin{tabular}{@{}ll
    D{.}{.}{3.2} D{.}{.}{2.2} D{.}{.}{1.3}
    D{.}{.}{1.3} D{.}{.}{1.3} D{.}{.}{1.3} D{.}{.}{1.3}
    D{.}{.}{2.2}
    D{.}{.}{2.2} D{.}{.}{2.2}@{}}
    \toprule
    \multirow{2}{*}{Optimizer} & \multirow{2}{*}{Step} &
    \multicolumn{3}{c}{Fidelity} & \multicolumn{4}{c}{Diversity} &
    \multicolumn{1}{c}{Alignment} & \multicolumn{2}{c}{Compositionality} \\
    \cmidrule(lr){3-5}\cmidrule(lr){6-9}\cmidrule(lr){10-10}\cmidrule(lr){11-12}
    & & \multicolumn{1}{c}{FD-DINO$\downarrow$} & \multicolumn{1}{c}{FID$\downarrow$} & \multicolumn{1}{c}{MMD-DINO$\downarrow$} &
    \multicolumn{1}{c}{Precision$\uparrow$} & \multicolumn{1}{c}{Recall$\uparrow$} & \multicolumn{1}{c}{Coverage$\uparrow$} & \multicolumn{1}{c}{Density$\uparrow$} &
    \multicolumn{1}{c}{HPSv2$\uparrow$} & \multicolumn{1}{c}{GenEval2 AM$\uparrow$} & \multicolumn{1}{c}{GenEval2 GM$\uparrow$} \\
    \midrule
    \multirow{12}{*}{AdamW} & 5k & 175.31 & 8.92 & 0.136 & 0.879 & 0.561 & 0.753 & 0.903 & 16.03 & 16.81 & 1.50 \\
     & 10k & 74.84 & 7.70 & 0.035 & 0.927 & 0.805 & 0.895 & 0.970 & 19.01 & 25.75 & 3.08 \\
     & 15k & 53.21 & 6.42 & 0.023 & 0.939 & 0.856 & 0.920 & 0.970 & 20.82 & 32.86 & 5.35 \\
     & 20k & 56.21 & 7.03 & 0.028 & 0.935 & 0.870 & 0.925 & 0.987 & 18.66 & 35.58 & 5.70 \\
     & 25k & 40.75 & 5.86 & 0.017 & 0.943 & 0.894 & 0.942 & 0.985 & 21.99 & 40.27 & 7.44 \\
     & 30k & 39.59 & 6.15 & 0.017 & 0.947 & 0.890 & 0.942 & 0.999 & 21.98 & 42.03 & 7.87 \\
     & 35k & 38.46 & 5.86 & 0.017 & 0.947 & 0.894 & 0.949 & 1.007 & 21.30 & 45.31 & 9.43 \\
     & 40k & 41.33 & 6.37 & 0.022 & 0.932 & 0.910 & 0.940 & 0.937 & 22.13 & 47.10 & 11.11 \\
     & 45k & 38.17 & 6.00 & 0.018 & 0.949 & 0.906 & 0.952 & 1.006 & 21.72 & 46.77 & 8.96 \\
     & 50k & 40.48 & 6.25 & 0.021 & 0.950 & 0.914 & 0.950 & 0.981 & 22.18 & 47.49 & 8.77 \\
     & 55k & 40.13 & 6.28 & 0.021 & 0.951 & 0.905 & 0.954 & 0.975 & 22.21 & 47.21 & 9.55 \\
     & 60k & 40.23 & 6.24 & 0.022 & 0.945 & 0.909 & 0.947 & 0.964 & 21.99 & 47.34 & 9.18 \\
    \midrule
    \multirow{12}{*}{Muon} & 5k & 119.42 & 7.83 & 0.088 & 0.903 & 0.694 & 0.846 & 0.923 & 17.16 & 21.96 & 2.26 \\
     & 10k & 54.49 & 6.35 & 0.025 & 0.936 & 0.865 & 0.928 & 0.965 & 20.23 & 33.89 & 5.72 \\
     & 15k & 37.89 & 5.82 & 0.015 & 0.942 & 0.898 & 0.941 & 0.951 & 22.10 & 43.35 & 8.73 \\
     & 20k & 37.24 & 5.41 & 0.016 & 0.944 & 0.909 & 0.940 & 0.947 & 21.96 & 46.38 & 9.81 \\
     & 25k & 33.48 & 5.15 & 0.014 & 0.947 & 0.914 & 0.945 & 0.952 & 22.92 & 54.30 & 14.21 \\
     & 30k & 37.13 & 5.84 & 0.018 & 0.942 & 0.915 & 0.947 & 0.937 & 20.93 & 52.72 & 11.44 \\
     & 35k & 35.92 & 5.92 & 0.018 & 0.943 & 0.917 & 0.950 & 0.961 & 20.87 & 52.67 & 10.67 \\
     & 40k & 33.55 & 5.37 & 0.018 & 0.952 & 0.924 & 0.954 & 0.952 & 22.88 & 54.66 & 12.15 \\
     & 45k & 33.25 & 5.19 & 0.018 & 0.945 & 0.921 & 0.950 & 0.948 & 22.53 & 57.50 & 13.76 \\
     & 50k & 34.54 & 5.74 & 0.020 & 0.941 & 0.925 & 0.953 & 0.934 & 22.67 & 58.18 & 14.11 \\
     & 55k & 33.47 & 5.62 & 0.019 & 0.946 & 0.921 & 0.949 & 0.949 & 22.43 & 58.53 & 13.84 \\
     & 60k & 33.51 & 5.57 & 0.019 & 0.945 & 0.930 & 0.950 & 0.937 & 22.37 & 57.93 & 14.55 \\
    \midrule
    \multirow{12}{*}{\shortstack[l]{Periodic Row-wise\\Muon}} & 5k & 137.51 & 7.68 & 0.101 & 0.899 & 0.641 & 0.817 & 0.917 & 17.49 & 19.15 & 2.61 \\
     & 10k & 56.18 & 6.37 & 0.022 & 0.936 & 0.851 & 0.918 & 0.971 & 19.37 & 26.87 & 3.44 \\
     & 15k & 45.14 & 6.64 & 0.019 & 0.939 & 0.885 & 0.938 & 0.956 & 21.86 & 34.76 & 5.73 \\
     & 20k & 40.13 & 6.34 & 0.017 & 0.941 & 0.901 & 0.939 & 0.963 & 21.81 & 39.12 & 6.93 \\
     & 25k & 32.35 & 5.61 & 0.012 & 0.949 & 0.912 & 0.947 & 0.966 & 22.90 & 49.68 & 11.74 \\
     & 30k & 35.12 & 5.58 & 0.016 & 0.949 & 0.907 & 0.953 & 0.994 & 20.59 & 47.55 & 10.70 \\
     & 35k & 36.91 & 6.19 & 0.017 & 0.944 & 0.914 & 0.943 & 0.959 & 19.94 & 48.71 & 9.45 \\
     & 40k & 38.72 & 6.02 & 0.021 & 0.943 & 0.916 & 0.945 & 0.942 & 20.43 & 45.89 & 9.19 \\
     & 45k & 34.37 & 5.60 & 0.018 & 0.941 & 0.919 & 0.951 & 0.954 & 21.08 & 53.10 & 12.08 \\
     & 50k & 52.25 & 7.86 & 0.039 & 0.944 & 0.916 & 0.941 & 0.923 & 19.64 & 44.64 & 8.98 \\
     & 55k & 40.25 & 6.38 & 0.026 & 0.942 & 0.915 & 0.949 & 0.945 & 20.80 & 47.80 & 10.12 \\
     & 60k & 33.99 & 5.63 & 0.020 & 0.948 & 0.925 & 0.955 & 0.967 & 21.35 & 52.26 & 11.97 \\
    \bottomrule
  \end{tabular}%
  }
  \vspace{2pt}
\end{table*}

\begin{figure*}[]
    \centering

    \begin{subfigure}[t]{\textwidth}
        \centering
        \includegraphics[width=0.82\textwidth]
            {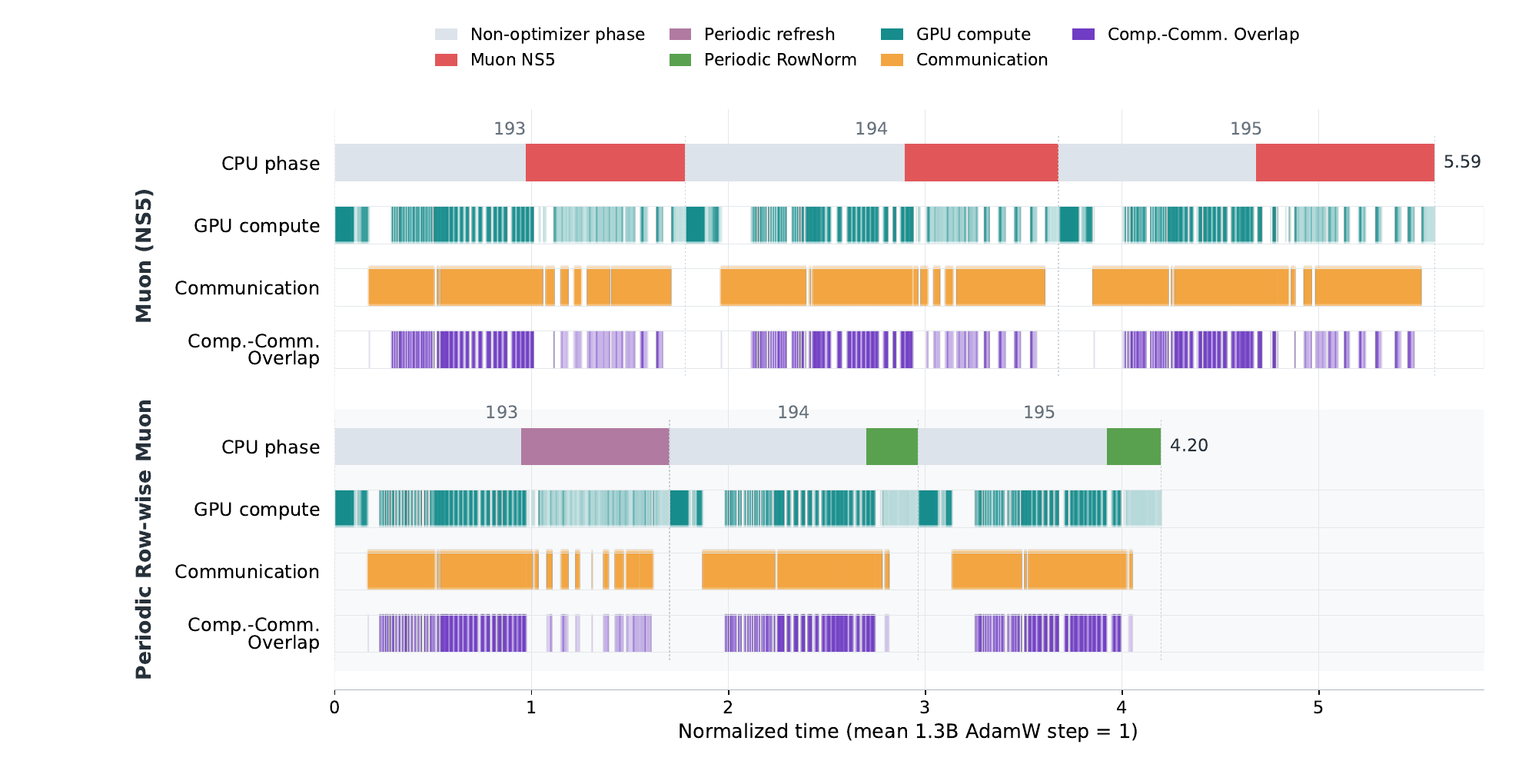}
        \caption{1.3B model.}
        \label{fig:profiler_1p3b}
    \end{subfigure}

    \vspace{0.1em}

    \begin{subfigure}[t]{\textwidth}
        \centering
        \includegraphics[width=0.82\textwidth]
            {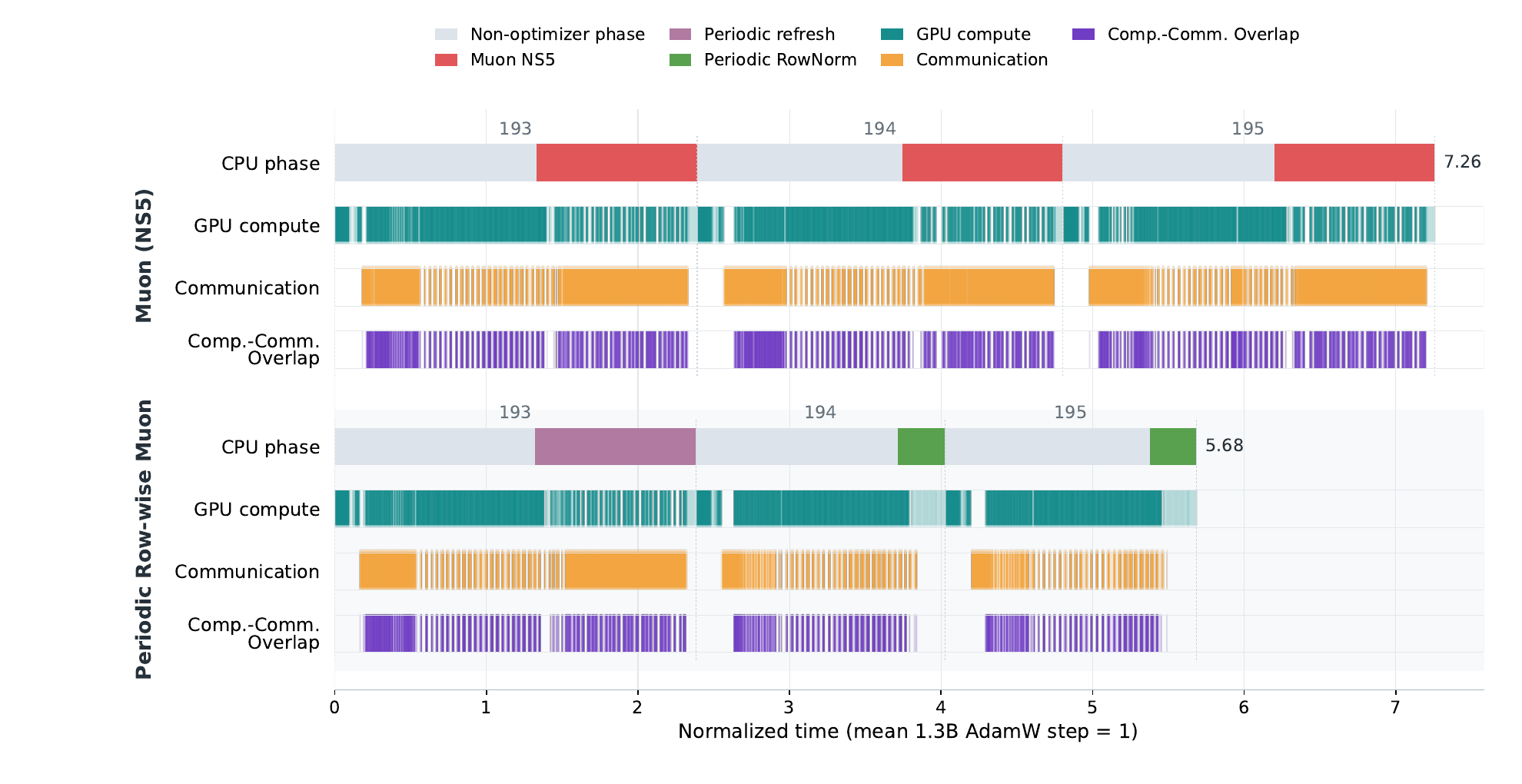}
        \caption{4B model.}
        \label{fig:profiler_4b}
    \end{subfigure}

    \vspace{0.1em}

    \begin{subfigure}[t]{\textwidth}
        \centering
        \includegraphics[width=0.82\textwidth]
            {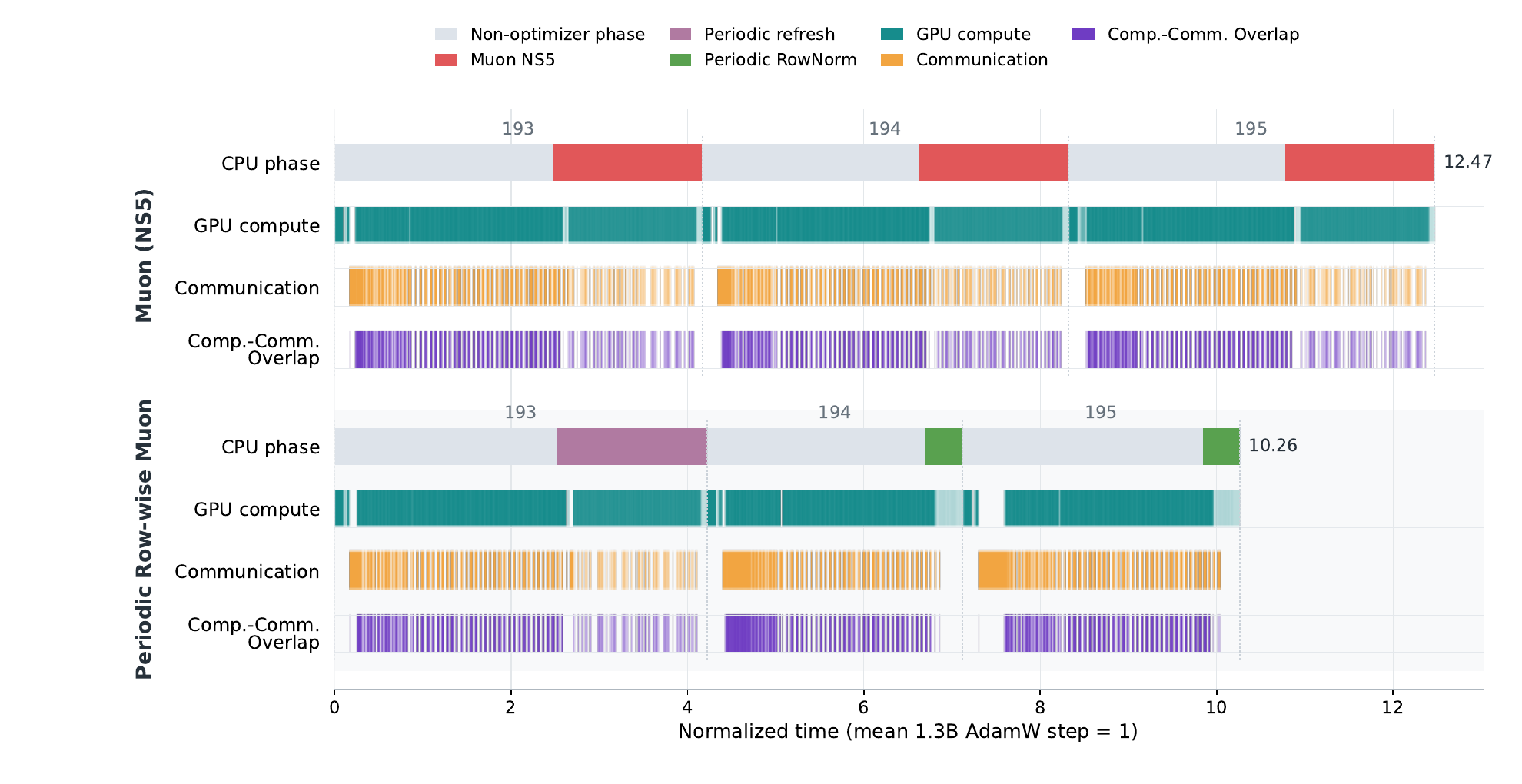}
        \caption{9B model.}
        \label{fig:profiler_9b}
    \end{subfigure}

    \caption{
        Simplified profiler traces for the 1.3B, 4B, and 9B models.
        The corresponding 15B result is shown in
        Figure \ref{fig:profiler_15b}.
    }
    \label{fig:profiler_additional}
\end{figure*}

\end{document}